\documentclass[11pt]{article}
\usepackage[paper=letterpaper,margin=1in]{geometry}

\usepackage{microtype}
\usepackage{graphicx}
\usepackage{booktabs} %

\usepackage{enumitem}
\usepackage{amsmath,amssymb,dsfont,amsthm,dsfont,hyperref}
\usepackage{xfrac}
\usepackage{xcolor}
\usepackage[title,toc,titletoc,page]{appendix}
\usepackage{natbib}
\usepackage[capitalize]{cleveref}
\usepackage{subcaption}
\usepackage{mathtools}

\usepackage{forloop}

\usepackage{thmtools} 
\usepackage{thm-restate}

\newtheorem{theorem}{Theorem}
\newtheorem{lemma}[theorem]{Lemma}
\newtheorem{cor}[theorem]{Corollary}
\newtheorem{obs}[theorem]{Observation}

\crefname{appsec}{Appendix}{Appendices}

\theoremstyle{definition}

\newcommand{\defcal}[1]{\expandafter\newcommand\csname 
	c#1\endcsname{{\mathcal{#1}}}}
\newcommand{\defbb}[1]{\expandafter\newcommand\csname 
	b#1\endcsname{{\mathbb{#1}}}}
\newcommand{\defbf}[1]{\expandafter\newcommand\csname 
	bf#1\endcsname{{\mathbf{#1}}}}
\newcounter{calBbCounter}
\forLoop{1}{26}{calBbCounter}{
	\edef\letter{\Alph{calBbCounter}}
	\expandafter\defcal\letter
	\expandafter\defbb\letter
	\expandafter\defbf\letter
}
\forLoop{1}{26}{calBbCounter}{
	\edef\letter{\alph{calBbCounter}}
	\expandafter\defbf\letter
}

\newcommand{\eps}{\varepsilon}
\newcommand{\E}{\mathbb{E}}

\DeclareMathOperator{\sign}{sign}
\DeclareMathOperator{\diag}{diag}

\newcommand{\ie}{i.e.}

\newcommand{\lc}[1]{}

\DeclareMathOperator{\unif}{Unif}
\DeclareMathOperator{\pois}{Poisson}
\DeclareMathOperator{\ber}{Ber}
\DeclareMathOperator{\bin}{Bin}
\DeclareMathOperator*{\argmax}{arg\,max}
\DeclareMathOperator*{\argmin}{arg\,min}

\newcommand{\wstd}{w^\textnormal{std}}
\newcommand{\wrob}{w^\textnormal{rob}}
\newcommand{\ltest}{L_{\textnormal{test}}}

\title{More Data Can Expand the Generalization Gap Between Adversarially Robust and Standard Models}
\author{Lin Chen\thanks{Department of Electrical Engineering, Yale University. E-mail: \texttt{linchen.dr@gmail.com}. First two authors contributed equally.} \and 
Yifei Min\thanks{Department of Statistics and Data Science, Yale University. E-mail: \texttt{yifei.min@yale.edu}.} \and 
Mingrui Zhang\thanks{Department of Statistics and Data Science, Yale University. E-mail: \texttt{mingrui.zhang@yale.edu}.}\and 
Amin Karbasi\thanks{Department of Electrical Engineering, Yale University. E-mail: \texttt{amin.karbasi@yale.edu}.}}
\date{}
\begin{document}
\maketitle

\begin{abstract}

    Despite remarkable success in practice, modern machine learning models have been found to be susceptible to adversarial attacks that make human-imperceptible perturbations to the data, but result in serious and potentially dangerous prediction errors. To address this issue, practitioners often use adversarial training to learn models that are robust against such attacks at the cost of higher generalization error on unperturbed test sets. The conventional wisdom is that more training data should shrink the gap between the generalization error of adversarially-trained models and standard models. However, we study the training of robust classifiers for both Gaussian and Bernoulli models under $\ell_\infty$ attacks, and we prove that more data may actually increase this gap. Furthermore, our theoretical results identify if and when additional data will finally begin to shrink the gap. Lastly, we experimentally demonstrate that our results also hold for linear regression models, which may indicate that this phenomenon occurs more broadly.

\end{abstract}

	\section{Introduction}
	
	As modern machine learning models continue to gain traction in the real world, a wide variety of novel problems have come to the forefront of the research community. One particularly important challenge has been that of adversarial attacks \citep{szegedy2013intriguing, goodfellow2014explaining, kos2018adversarial, carlini2018audio}.
	To be specific, given a model with excellent performance on a standard data set, one can add small perturbations to the test data that can fool the model and cause it to make wrong predictions. What is more worrying is that these small perturbations can possibly be designed to be imperceptible to human beings, which raises concerns about potential safety issues and risks, especially when it comes to applications such as autonomous vehicles where human lives are at stake. %
	
	The problem of adversarial robustness in machine learning models has been explored from several different perspectives since its discovery. One direction has been to propose attacks that challenge these models and their training procedures \citep{gu2014towards, moosavi2016deepfool, papernot2016limitations, carlini2017adversarial, athalye2018obfuscated}. In response, there have been works that propose more robust training techniques that can defend against these adversarial attacks %
	\citep{he2017adversarial, raghunathan2018certified, raghunathan2018semidefinite, shaham2018understanding, weng2018towards, wong2018provable, zhang2018efficient, cohen2019certified, lecuyer2019certified, stutz2020confidence}. For robust training, one promising approach is to treat the problem as a minimax optimization problem, where we try to select model parameters that minimize the loss function under the strongest feasible perturbations \citep{xu2012robustness, madry2017towards}. Overall, adversarial training may be computationally expensive \citep{bubeck2019adversarial, nakkiran2019adversarial}, but it can lead to enhanced resistance towards adversarially modified inputs.

	Although adversarially robust models tend to outperform standard models when it comes to perturbed test sets, recent studies have found that such robust models are also likely to perform worse on standard (unperturbed) test sets \citep{raghunathan2019adversarial, tsipras2018robustness}. %
	We refer to the difference in test loss on unperturbed test sets as the cross generalization gap. This paper focuses on the question of whether or not this gap can be closed.
	
	Theoretical work by \citet{schmidt2018adversarially} has shown that adversarial models require far more data than their standard counterparts to reach a certain level of test accuracy. This supports the general understanding that adversarial training is harder than standard training, as well as the conventional wisdom that more data helps with generalization. However, when it comes to the cross generalization gap, things may not be so simple.

	In this paper, we identify two regimes during the adversarial training process. In one regime, more training data eventually helps to close the cross generalization gap, as expected. In the other regime, the gap will surprisingly continue to grow as more data is used in training. The data distribution and the strength of the adversary determine the regime and the %
	existence of the two regimes indicates a fundamental phase transition in adversarial training.

	\subsection{Our Contributions}
	
	In our analysis of the cross generalization gap, %
	we assume the robust model is trained under $\ell_\infty$ constrained perturbations. We study two classification models including a Gaussian model and a Bernoulli model, as well as a simple linear regression model. 
	
	For the Gaussian model, we theoretically prove that during the training of a robust classifier there are two possible regimes that summarize the relation between the cross generalization gap and the training sample size (see \cref{thm:gaussian}). More specifically, let $n$ denote the number of training data points. Suppose the perturbation that the adversary can add is constrained to the $\ell_\infty$ ball of radius $\eps$. In the strong adversary regime (i.e. large $\eps$ compared to the signal strength of the data), the gap always increases and has an infinite data limit. 
	
	In contrast, in the weak adversary regime, there exists a critical point that marks the boundary between two stages. For all $n$ less than this threshold, we have the increasing stage where the gap monotonically increases. %
	Beyond this threshold, we will eventually 
	reach another stage where the gap strictly decreases. It is important to note that, even in the weak adversary regime, it is 
	possible to make this threshold arbitrarily large, which means adding data 
	points will always expand the cross generalization gap. 
	
	For the Bernoulli model, we show similar results (see \cref{thm:bernoulli}). %
	Although the curve for the cross generalization gap will be oscillating (see \cref{fig:bernoulli}), we prove that it manifests in a general increasing or decreasing trend. We further explore a simple one-dimensional linear regression and experimentally verify that the phase transition also exists.
	
	The primary implication of our work is that simply adding more data will not always be enough to close the cross generalization gap. Therefore, fundamentally new ideas may be required if we want to be able to train adversarially robust models that do not sacrifice accuracy on unperturbed test sets.

	\section{Related Work}
	There is an existing body of work studying adversarially robust models and their generalization. We briefly discuss some of the papers that are most relevant to our work.
	
	\paragraph{Trade-off between robustness and standard accuracy} What initially motivated our work is the experimental finding that standard accuracy and adversarial robustness can sometimes be incompatible with each other \citep{papernot2016towards, tsipras2018robustness}. These works empirically show that using more data for adversarial training might decrease the standard accuracy. Additionally, this decline becomes more obvious when the radius of perturbation $\eps$ increases. This causes the cross generalization gap between robust and standard models. 
	The side effect of a large perturbation has also been studied by \citet{dohmatob2019generalized} who shows that it is possible to adversarially fool a classifier with high standard accuracy if $\eps$ is large. 
	\citet{ilyas2019adversarial} explore the relation between the perturbation $\eps$ and the features learned by the robust model. Their results suggest that a larger $\eps$ tends to add more weight onto non-robust features and consequently the model may miss useful features which should be learned under standard setting. \citet{diochnos2018adversarial} consider both error region setting and study the classification problem where data is uniformly distributed over $\left\{ 0,1 \right\}^d$. They show that under this $\ell_0$ perturbation setting the adversary can fool the classifier into having arbitrarily low accuracy with at most $\eps = O(\sqrt{d})$ perturbation. \citet{zhang2019theoretically} theoretically study the trade-off between robustness and standard accuracy from a perspective of decomposition. More specifically, they decompose the robust error into a standard error and a boundary error that would be affected by the perturbation. Their decomposition further leads to a new design of defense. Empirically, to deal with the reduction in the standard accuracy, \citet{stutz2019disentangling} show that if the perturbation is not large enough to push data points across the decision boundary and the resulting adversarial examples still stay within their true decision region, then the adversarial training with such examples can boost generalization. \citet{zhang2020attacks} also propose training on specifically chosen adversarial examples to reduce the drop in the standard accuracy. Brittleness/robustness of Bayesian Inference is studied by \citet{owhadi2013brittleness, owhadi2015brittleness1, owhadi2015brittleness2, owhadi2017qualitative}. 
	
	In a concurrent and independent work, \citet{raghunathan2020understanding} performed a finite-sample analysis of the trade-off for a linear regression model. They also leveraged the recently proposed robust self-training estimator \citep{carmon2019unlabeled, najafi2019robustness} in order to mitigate the robust error without sacrificing the standard error. 
	They focused on a regression problem on the original training dataset augmented with perturbed examples and investigated a regime where the optimal predictor has zero standard and robust error. 
	This paper studied a classification problem and our analysis covers both weak and strong regimes.
	
	\paragraph{Sample complexity for generalization}
	The generalization of adversarially robust models has different 
	properties from the standard ones, especially in sample complexity. \citet{schmidt2018adversarially} study Gaussian mixture models in $d$-dimensional space and show that 
	for the standard model only a constant number of training data points is needed, while for the 
	robust model under $\ell_\infty$ perturbation a training dataset of size $\Omega(d)$ is required. Their work is in a different direction 
	to ours: their main 
	result focuses on dimension-dependent bounds for sample complexity, while we quantify the effect of the amount of training data on adversarial 
	generalization and we prove the existence of a phase transition under two 
	binary classification models. \citet{bubeck2019adversarial} analyze the 
	computational hardness in training a robust classifier in the statistical 
	query model. They prove that for a binary classification problem in $d$ 
	dimensions, one needs polynomially (in $d$) many queries to train a standard
	classifier while exponentially many queries to train a robust one. \citet{garg2020adversarially} consider a setting where the adversary has limited computational power and show that there exist learning tasks that can only be robustly solved when faced with such limited adversaries. %
	\citet{yin2019rademacher} and \citet{ khim2018adversarial} prove 
	generalization bounds for linear classifiers and neural networks via 
	Rademacher complexity. In addition, \citet{yin2019rademacher} show the adversarial 
	Rademacher complexity is always no less than the standard one and is 
	dimension-dependent. \citet{montasser2019vc} show the widely used uniform convergence of empirical risk minimization framework, or more generally, any proper learning rule, might not be enough for robust generalization. They prove the existence of a hypothesis class where any proper learning rule gives poor robust generalization accuracy under the PAC-learning setting, while improper learning can robustly learn any class. \citet{cullina2018pac} study generalization under the PAC-learning setting and prove a polynomial upper bound for sample complexity that depends on a certain adversarial VC-dimension. \citet{diochnos2019lower} study PAC-learning under the error region setting and prove a lower bound for sample complexity that is exponential in the input dimension. 
	
	\paragraph{Other relevant work} \citet{bhagoji2019lower} use optimal transport to derive lower bounds for the adversarial classification error. For a binary classification problem, they prove a relation between the best possible adversarial robustness and the optimal transport between the two distributions under a certain cost. 
	Another line of work analyzes adversarial examples via concentration of measure and show that their existence is inevitable under certain conditions \citep{gilmer2018adversarial, fawzi2018adversarial, shafahi2018adversarial, mahloujifar2019curse}.

	\section{Preliminaries}
	
	\subsection{Notation}\label{sub:notation}
	We use the shorthand $ [d] $ to denote the set $ \{1,2,\dots,d\} $ for any positive integer $ d $. %
	We use $ \cN(\mu,\Sigma) $ to denote the 
	multivariate Gaussian distribution with mean vector $ \mu $ and covariance 
	matrix $ \Sigma $. 
	
	If $ u, v\in \bR^d $ are two $ d $-dimensional vectors, the $ j $-th 
	component of $ u $ is denoted by $ u(j) 
	$. The inner product of $ u $ and $ v $ is denoted by $ \langle u, v\rangle 
	$. 
	If $ A $ is a positive semi-definite matrix, let the semi-norm induced by $ 
	A $ be $ \|u\|_A = \sqrt{u^\top A u} $.
	Let $ B_u^\infty(\eps) $ denote the $ \ell_\infty $ ball centered at 
	$ u $ and with radius $ \eps $, \ie, $ B_u^\infty(\eps) = \{ v\in \bR^d: \| 
	u-v \|_\infty \le \eps \} $. In our problem setup in 
	\cref{sub:problem_setup}, the 
	ball $ B_u^\infty $ is the set of allowed perturbed vectors for the 
	adversary, where $ \eps $ is the perturbation budget. 
	We define the Heaviside step function $ H $ to be \[ 
	H(x) = \begin{cases}
	1, & \text{for }x > 0\,;\\
	\sfrac{1}{2}, & \text{for } x = 0\,;\\
	0, & \text{for } x < 0\,.
	\end{cases}
	 \]
	
	\subsection{Problem Setup}\label{sub:problem_setup}
	
	Suppose that the data $(x,y)$ is drawn from an 
	unknown distribution $\cD$, where $x$ is the input and $y$ is the  
	label. 
	For example, in a classification problem, we have $(x,y)\in \bR^d\times 
	\{\pm 1\}$; in a regression problem, we have $ (x,y)\in \bR^d\times \bR $. 
	 Given a model parameter $ w \in \Theta\subseteq \bR^{p} $ and a data point 
	 $(x,y)$, 
	 the loss of the model parameterized by $ w $ on the data point $ (x,y) $ 
	 is denoted by $ \ell(x,y;w) $. 
	 
	 \newcommand{\dtrain}{D_\textnormal{train}}
	 \newcommand{\dtest}{D_\textnormal{test}}
	The training dataset $ \dtrain= \{(x_i,y_i)\}_{i=1}^n$ consists of $n$ data 
	points sampled i.i.d.\ from the distribution $\cD$. Given the training 
	dataset with size $n$,
	we respectively define the optimal standard and robust 
	models trained on $ \dtrain $ by
	\begin{equation} \label{eq:wstd_wrob_general}
	\begin{split}
	\wstd_n ={}& \argmin_{w\in \Theta} \frac{1}{n}\sum_{i=1}^n 
	\ell(x_i,y_i;w)\,,\\
	\wrob_n = {}& \argmin_{w\in \Theta} \frac{1}{n} \sum_{i=1}^n
	\max_{\tilde{x}_i\in B^\infty_{x_i}(\eps)}
	\ell(\tilde{x}_i,y_i;w)\,.\\
	\end{split}
	\end{equation}
	The optimal standard model $ \wstd $ is the minimizer of the total 
	training loss 
	$ \frac{1}{n}\sum_{i=1}^n 
	\ell(x_i,y_i;w) $. In the definition of the optimal robust model 
	$ \wrob $, we take into consideration the adversarial training for each 
	data point, i.e., the inner maximization $ \max_{\tilde{x}_i\in 
		B^\infty_{x_i}(\eps)}
	\ell(\tilde{x}_i,y_i;w) $. We assume that the 
	adversary is able to perturb each data item $ x_i $ within an $ 
	\ell_\infty $ ball centered at $ x_i $ and with radius $ \eps $. 
	The best robust model is the minimizer of the total training loss with 
	adversarial training. 
	Note that both $ \wstd $ and $ \wrob $ are functions 
	of the training dataset and thereby also random variables. 
	
	If we have a model parametrized $ w $ and the test dataset $ 
\dtest=	\{(x'_i,y'_i)\}_{i=1}^{n'} $ consists of $ n' $ data points sampled 
i.i.d.\ 
	from $ \cD $, the test loss of $ w $ is given by \begin{equation*}
	\ltest(w) ={} \bE\left[ \frac{1}{n'}\sum_{i=1}^{n'} 
	\ell(x'_i,y'_i;w) \right]
	={}\bE_{(x,y)\sim \cD}\left[ \ell(x,y;w) \right]\,.
	\end{equation*}
	
		Additionally, we define the cross generalization gap $ g_n $ between the 
	standard and robust classifiers by
	\begin{equation*} %
	\begin{split}
	g_n ={}& \bE_{\{(x_i,y_i)\}_{i=1}^n\stackrel{\textnormal{i.i.d.}}{\sim} 
		\cD}\left[ \ltest(\wrob) - \ltest(\wstd) \right]\\
	={} &
	\bE_{\{(x_i,y_i)\}_{i=1}^n\stackrel{\textnormal{i.i.d.}}{\sim} 
		\cD} \left[ \bE_{(x,y)\sim \cD}[\ell(x,y;\wrob)]\right.
	\left.- \bE_{(x,y)\sim \cD}[\ell(x,y;\wstd)]\right]\,.
	\end{split}
	\end{equation*}
	
	\section{Classification}\label{sec:classification}
	
	In this section, we study a binary classification problem, where we have 
	each data point $ (x,y) \in \bR^d \times \{\pm 1\} $. For any model 
	parameter $ w\in \bR^d $, we 
	consider the loss function $\ell(x,y;w) = -y\langle 
	w,x\rangle$~\citep{yin2019rademacher,khim2018adversarial}. The parameter $ 
	w $ is constrained on the $ \ell^\infty $ ball $ \Theta = \{ w\in \bR^d \mid
	\|w\|_\infty \le W \} $, where $W$ is some positive real number.  Under this 
	setup, the best standard and robust classifier are given as follows. 
	\begin{equation}\label{eq:wstd_wrob}
	\begin{split}
	\wstd_n ={}& \argmin_{\|w\|_\infty\le W} \frac{1}{n}\sum_{i=1}^n 
	-y_i\langle w,x_i\rangle
	={} \argmax_{\|w\|_\infty\le W} \sum_{i=1}^n 
	y_i\langle w,x_i\rangle\,, \\
	\wrob_n = {}& \argmin_{\|w\|_\infty\le W} \frac{1}{n} \sum_{i=1}^n
	\max_{\tilde{x}_i\in B^\infty_{x_i}(\eps)}
	\left( -y_i\langle w,\tilde{x_i}\rangle \right)
	= {} \argmax_{\|w\|_\infty\le W} \sum_{i=1}^n 
	\min_{\tilde{x}_i\in B^\infty_{x_i}(\eps)}
	y_i\langle w,\tilde{x_i}\rangle\,.
	\end{split}
	\end{equation}

	The cross generalization gap $ g_n $ between the 
	standard and robust classifiers is given by
	\begin{equation}\label{eq:generalization_gap}
	\begin{split}
	g_n 
	={} &
	\bE_{\{(x_i,y_i)\}_{i=1}^n\stackrel{\textnormal{i.i.d.}}{\sim} 
		\cD} \left[ \bE_{(x,y)\sim \cD}[y\langle 
	\wstd,x\rangle] - \bE_{(x,y)\sim \cD}[y\langle 
	\wrob,x\rangle]\right]\,.
	\end{split}
	\end{equation}

	In this paper, we investigate how the cross generalization gap $ g_n $ evolves with the amount of data. Intuitively, one might conjecture that the gap 
	should satisfy the following properties: 
	\begin{enumerate}[label=(\alph*)]
		\item First, the gap should always be 
		non-negative. This means that the robust classifier incurs a larger 
		test 
		(generalization) loss than the standard classifier, as there is no free 
		lunch and robustness in adversarial training would compromise 
		generalization performance. \label{it:gn-positive}
		\item Second, more training data would close the gap gradually; in 
		other words, the gap would be decreasing with respect to the size of 
		the training dataset. \label{it:gn-decrease}
		\item Third, in the infinite data limit (\ie, when the size of the 
		training dataset tends to infinity), the cross generalization gap would 
		eventually tend to zero. \label{it:gn-zero}
	\end{enumerate}

	Our study corroborates \ref{it:gn-positive} but denies 
	\ref{it:gn-decrease}
	and \ref{it:gn-zero} in general. The implication of this is not only that current adversarial training techniques sacrifice standard accuracy in exchange for robustness, but that simply adding more data may not solve the problem.

	\begin{figure*}
		\begin{subfigure}{0.49\textwidth}
			\centering
			\includegraphics[width=\linewidth]{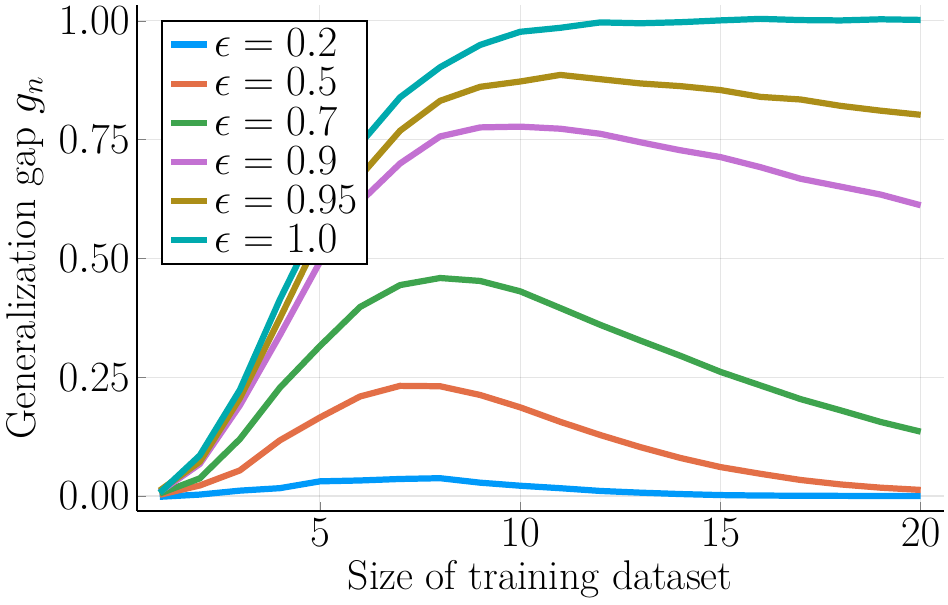}
			\caption{Gaussian model}
			\label{fig:gaussian}
		\end{subfigure}\hfill
		\begin{subfigure}{0.49\textwidth}
			\centering
			\includegraphics[width=\linewidth]{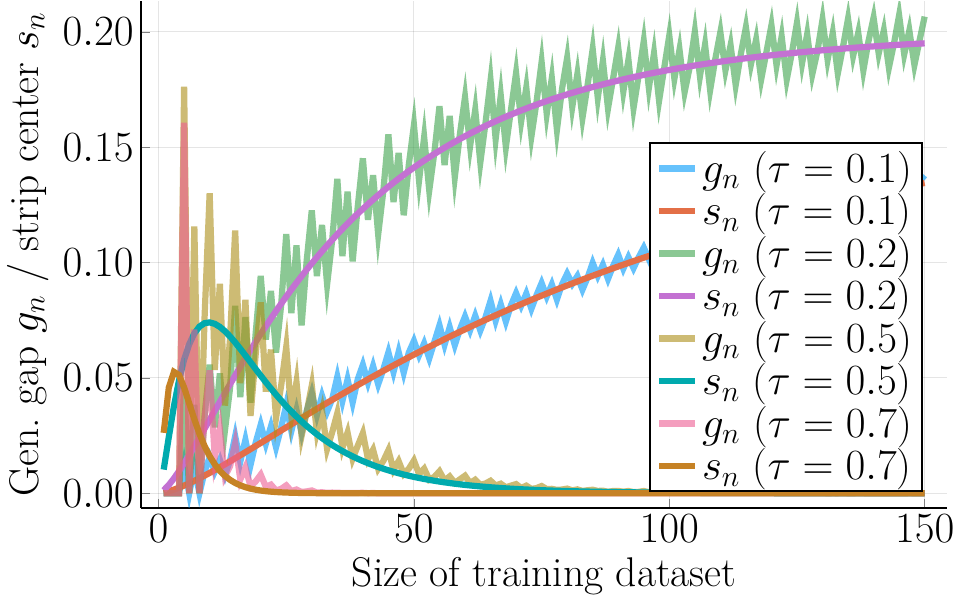}
			\caption{Bernoulli model}
			\label{fig:bernoulli}
		\end{subfigure}
	\caption{Cross generalization gap $ g_n $ (and strip center $ s_n $ for the 
	Bernoulli model) vs.\ the size of the training dataset.}
	\label{fig:classification}
	\end{figure*}
	
	\subsection{Gaussian Model}
	\newcommand{\dgau}{\cD_{\textnormal{Gau}}}
	\newcommand{\dber}{\cD_{\textnormal{Ber}}}
	The Gaussian model is specified as follows. 
	Let $ (x,y)\in \bR^d \times \{\pm 1\} $ obey the distribution such that 
	$ y\sim \unif(\{\pm 1\}) $ and $ x\mid y\sim \cN(y\mu, \Sigma) $, where 
	$ \mu(j)\ge 0 $ for $ \forall j\in [d] $ and 
	$ 
	\Sigma=\diag(\sigma(1)^2,\sigma(2)^2,\dots,\sigma(d)^2) $. We 
	denote this 
	distribution by $ (x,y)\sim \dgau $.
	
	\begin{theorem}[Gaussian model, \textbf{proof in 
	\cref{sec:proof-gaussian}}]\label{thm:gaussian}
		 Given i.i.d.\ training data $ 
		(x_i,y_i)\sim 
		\dgau $ with $ n $ data points, 
		if we define the standard and robust classifier as in 
		\eqref{eq:wstd_wrob} (denoted by $\wstd$ and $\wrob$, respectively) and 
		define the cross generalization gap $g_n$ as in \eqref{eq:generalization_gap},
		 we have \begin{enumerate}[label=(\alph*),nosep]
		 	\item $ g_n \ge 0 \hspace{0.05in} \forall \  n\ge 1 $; \label{it:gn_ge0}
		 	\item The infinite data limit equals \[ 
		 	\lim_{n\to\infty} g_n = 2W \sum_{j\in [d]:\mu(j)>0} \mu(j) H\left( 
		 	\frac{\eps}{\mu(j)}-1 \right)\, ,
		 	\] \label{it:gn_lim}
		 	where $H$ is the Heaviside step function defined in \cref{sub:notation};
		 	\item If $ \eps < \min_{j\in [d]:\mu(j)>0} \mu(j) $, $ g_n $ is 
		 	strictly increasing in $ n $ when 
		 	\[
		 	n <  
		 	\min_{\substack{j\in [d]:\\\mu(j)>0}} \max\left\{ 
		 	\frac{3}{2} ,2\log\frac{1}{1-\eps/\mu(j)} \right\} \left(  
		 	\frac{\sigma(j)}{\mu(j)} 
		 	\right)^2 \,,
		 	\]
and it is strictly decreasing in $ n $ when \begin{equation*}
n \ge \max_{\substack{j\in [d]:\\ \mu(j)>0}}\left( K_0 + 2\log \frac{1}{1-\eps/\mu(j)} 
\right)\left( \frac{\sigma(j)}{\mu(j)} \right)^2\,,
\end{equation*}
where $ K_0 $ is a universal constant.
		 	\label{it:gn_increase} 
		 	\item If $ \eps > \| \mu \|_\infty $, $ g_n $ is strictly 
		 	increasing for all $ n\ge 1 $.\label{it:gn_always_increase}
		 \end{enumerate} 
	\end{theorem}

Part \ref{it:gn_ge0} of \cref{thm:gaussian} states that the generalization of the robust classifier is never better than the standard one. Part \ref{it:gn_lim} quantifies the size of the gap as the size of the training dataset $n$ goes to infinity. The main implication here is that the gap will always converge to some finite limit, which may be zero if the strength of the adversary $\epsilon$ is small enough. 

Parts \ref{it:gn_increase} and \ref{it:gn_always_increase} describe the two different possible regimes. Part \ref{it:gn_increase} states that if the strength of the adversary is not too large, then there will be two stages: an initial stage where the cross generalization gap is strictly increasing in $n$, followed by a secondary stage where the gap is strictly decreasing in $n$. On the other hand, part \ref{it:gn_always_increase} states that a large $\epsilon$ will result in a cross generalization gap that is strictly increasing (but still tending towards some finite limit).

In order to better describe and visualize the implications of \cref{thm:gaussian}, we consider a 
special 
case where $ \mu = (\mu_0,\dots,\mu_0) $ and $ \Sigma = \sigma_0^2 I $. 
\begin{cor}\label{cor:gaussian}
	Assume that $ W=1 $, $ \mu(j) = \mu_0 \ge 0 $, and $ \sigma(j) = \sigma_0 > 
	0 $ for 
	all $ j\in [d] $. The infinite data limit equals \begin{equation*}
	\lim_{n\to\infty} g_n = 2d\mu_0 H\left( \frac{\eps}{\mu_0} - 1 \right)
	= \begin{cases}
	2d\mu_0, & \text{for } \frac{\eps}{\mu_0}>1\,;\\
	d \mu_0, & \text{for } \frac{\eps}{\mu_0}=1\,;\\
	0, & \text{for } \frac{\eps}{\mu_0}<1\,.
	\end{cases}
	\end{equation*}
	If $ \eps < \mu_0 $, we have $ g_n $ is strictly 
	increasing when \begin{equation*}
	n < \max\left  \{ \frac{3}{2}, 2\log \frac{1}{1-\eps/\mu_0} \right \}\left( 
	\frac{\sigma_0}{\mu_0} \right)^2\,,
	\end{equation*}
	 and it is strictly decreasing when \begin{equation*}
	 n \ge \left( K_0 + 2\log \frac{1}{1-\eps/\mu_0} \right) \left( 
	 \frac{\sigma_0}{\mu_0} \right)^2\,,
	 \end{equation*}
	 where $ K_0 $ is a universal constant.
	 If $ \eps > \mu_0 $, we have $ g_n $ is strictly increasing for all $ n\ge 
	 1 $. 
\end{cor}

\cref{cor:gaussian} is essentially a simplified version of parts \ref{it:gn_increase} and \ref{it:gn_always_increase} of \cref{thm:gaussian} where we cleanly divide between a weak adversary regime and a strong adversary regime at a threshold $\eps = \mu_0$.

We illustrate the cross generalization gap $ g_n $ vs.\ the size of the training 
dataset in \cref{fig:gaussian}, where we set $ W = d = \mu = 1 $ and $ \sigma = 
2 $. The curve $ \eps = 1 $ belongs to the strong adversary regime, while the 
remaining curves belong to the weak adversary regime. 

In the weak adversary regime, the evolution 
of $ g_n $ can be divided into two stages, namely the increasing and decreasing 
stages (part \ref{it:gn_increase} of \cref{thm:gaussian}). 
	The duration of the increasing 
stage is 
\begin{equation*}
\Theta \left( \left( \frac{\sigma_0}{\mu_0} 
\right)^2 \log\frac{1}{1-\eps/\mu_0} \right)\,.
\end{equation*}
 This duration is controlled by 
the ratio $ \eps/\mu_0 $, as well as the reciprocal of the 
signal-to-noise ratio (SNR), \ie, $ \frac{\sigma_0}{\mu_0} $. A larger 
SNR and an $ \eps $ closer to $ \mu_0 $ lead to a shorter 
increasing stage. It can be observed in \cref{fig:gaussian} that for the curves 
with $ \eps = 0.2,0.5,0.7,0.9,0.95 $, a larger $ \eps $ results in a longer 
duration of the increasing stage. 

After the increasing stage, the cross generalization gap will eventually begin to decrease towards some finite limit (given by part \ref{it:gn_lim} of \cref{thm:gaussian})  if sufficient training data is provided. In addition, we would like to remark that the duration relies on the data and 
the strength 
of the adversary and could be potentially arbitrarily large; in other words, 
without full information about the true data distribution and the power of the 
adversary, one cannot predict when the increasing stage will terminate.

In the strong adversary regime, the cross generalization gap expands from the very 
beginning. In the infinite data limit, the gap approaches $ d \mu_0 $ if $ \eps = \mu_0 $, and it approaches $ 2d \mu_0 $ if $ \eps > \mu_0 $.

	\subsection{Bernoulli Model}
	In this subsection, we investigate the Bernoulli model defined as follows. 
	Let $ (x,y)\in \bR^d \times \{\pm 1\} $ obey the distribution such that 
	$ y\sim \unif(\{\pm 1\}) $ and for $ \forall j \in [d] $ independently,
	\[
	x(j) = \begin{cases}
	y\cdot \theta(j) &\text{with probability } \frac{1+\tau}{2}\,,\\
	-y\cdot \theta(j) & \text{with probability } \frac{1-\tau}{2}\,,
	\end{cases}
	\]
	where $ \theta\in \bR^d_{\ge 0} $ and $ \tau\in (0,1) $.
	We 
	denote this 
	distribution by $ (x,y)\sim \dber $. 
	
	The parameter $ \tau $ controls the signal strength level. When $ \tau =0 $ 
	(lowest signal strength),  $ x(j) $ takes the value of $ +\theta(j) $ or $ -\theta(j) $ 
	uniformly at random, irrespective of the label $ y $. When $ \tau=1 $ 
	(highest signal strength), we have $ x(j) = y\cdot \theta(j) 
	$ almost surely. 
	
	We illustrate the cross generalization gap $ g_n $ vs.\ the size of the training 
	dataset (denoted by $ n $) in  \cref{fig:bernoulli}, where we set $ W = d = 
	\theta = 1 $ and $ \eps = 0.2 $. We observe that all curves $ g_n $ 
	oscillate around the other 
	curves labeled $ s_n $. Although the figure shows that the 
	curves $ g_n $ are not monotone, they all exhibit a monotone trend, which 
	is characterized by $ s_n $. 
	
	As a result, we will not show that $ g_n $ is monotonically increasing or 
	decreasing (as shown in \cref{fig:bernoulli}, it is not monotone). 
	Alternatively, we will show that $ g_n $ resides in a strip centered around 
	$ s_n $ and $ s_n $ displays (piecewise) monotonicity. Additionally, the 
	height of the strip shrinks at a rate of $ O\left( \frac{1}{\sqrt{n}} 
	\right) $; in other words, it can be shown that \begin{equation*}
	\left |g_n-s_n \right | \le O\left( 
	\frac{1}{\sqrt{n}} \right), \quad \forall n \ge 1\,.
	\end{equation*}

	\begin{theorem}[Bernoulli model, \textbf{proof in 
	\cref{sec:proof-bernoulli}}]\label{thm:bernoulli}
		 Given i.i.d.\ training data $ 
		(x_i,y_i)\sim 
		\dber $ with $ n $ data points, if we define the standard and robust 
		classifier (denoted by $ \wstd $ and $ \wrob $, respectively) as in 
		\eqref{eq:wstd_wrob}
		and define the cross generalization gap $ g_n $ as in 
		\eqref{eq:generalization_gap}, 
		we have
		\begin{enumerate}[nosep,label=(\alph*)]
			\item $ g_n\ge 0 $ for $ \forall n\ge 1 $;
			\item The infinite data limit equals \[ 
			\lim_{n\to\infty} g_n = 2W\tau \sum_{j\in [d]:\theta(j)>0} 
			\theta(j) 
			H\left(\frac{\eps}{\theta(j)\tau}-1\right)\,,
			\] 
			where $ H $ is the Heaviside step function defined in \cref{sub:notation}.  
		\end{enumerate}
Furthermore, 
	there exists a positive constant $ C_0 \le 
	\frac{\sqrt{10}+3}{6\sqrt{2\pi}}\approx 0.4097 $ and a sequence $ s_n $ 
	such that $ |g_n-s_n|\le \frac{8C_0W\tau \| \theta \|_1(\tau 
		^2+1)}{\sqrt{n}\sqrt{1-\tau ^2}}  $  and
		 \begin{enumerate}[resume,nosep,label=(\alph*)]
			\item If $ \frac{\eps}{\tau} < \min_{j\in [d]:\theta(j)>0} 
			\theta(j) $,   $ s_n $ is strictly increasing in $ n $ when $$ n < 
			\left ( \frac{1}{\tau^2}-1 \right ) \max \left \{ \frac{3}{2}, 
			2\min_{\substack{j\in 
			[d]:\\ \theta(j)>0}} \log \frac{1}{1-\frac{\eps}{\theta(j)\tau}} \right \}  
			$$ and strictly decreasing in $ n $ when \begin{equation*}
			 n \ge \left( \frac{1}{\tau^2}-1 \right) \left( K_0 + 
			2 \max_{\substack{j\in [d]:\\ \theta(j)>0}} \log 
			\frac{1}{1-\frac{\eps}{\theta(j)\tau}} 
			\right) \,,
			\end{equation*}
			where $ K_0 $ is a universal constant;
			\item If $ \frac{\eps}{\tau}\ge \| \theta \|_{\infty} $, 
			$ s_n $ is strictly increasing for all  $ n \ge 1
			$. 
		\end{enumerate} 
	\end{theorem}

Again, to explain the implications of \cref{thm:bernoulli}, we explore the 
following special case where $ W=1 $ and $ \theta = (\theta_0,\dots,\theta_0) $.
\begin{cor}
	Assume $ W=1 $ and that $ \theta(j) = \theta_0 > 0 $ holds for all $ j\in 
	[d] $. The infinite data limit equals \begin{equation}\label{eq:ber-gn}
	\begin{split}
	\lim_{n\to\infty} g_n ={}& 2\tau d \theta_0 H\left( 
	\frac{\eps}{\theta_0\tau}  
	-1 \right)\\
	={}& \begin{cases}
	2\tau d \theta_0, & \text{for } \eps > \theta_0 \tau\,;\\
	\tau d \theta_0, & \text{for } \eps = \theta_0 \tau\,;\\
	0, & \text{for } \eps < \theta_0 \tau\,.
	\end{cases}
	\end{split}
	\end{equation}
	If $ \eps < \theta_0\tau $, $ s_n $ is strictly increasing in $ n $ when 
	\begin{equation*}
	n < 
	\left ( \frac{1}{\tau^2}-1 \right ) \max \left \{ \frac{3}{2}, 
	2 \log \frac{1}{1-\eps/(\theta_0 \tau)} \right \} \,,
	\end{equation*}
	and it is strictly decreasing when \begin{equation*}
	n \ge \left( \frac{1}{\tau^2}-1 \right) \left( K_0 + 
	2  \log 
	\frac{1}{1-\eps/(\theta_0 \tau)} 
	\right) \,,
	\end{equation*}
	where $ K_0 $ is a universal constant. If $ \eps \ge \theta_0\tau $, $ s_n 
	$ is strictly increasing for all $ n\ge 1 $. 
\end{cor}

Similar to the Gaussian model, there also exist two regimes. One is the weak 
adversary regime where $ \eps < \theta_0 \tau $, while the other is the strong 
adversary regime where $ \eps \ge \theta_0\tau $. Recall that in 
\cref{fig:bernoulli}, we set $ W=d=\theta=1 $ and $ \eps = 0.2 $. Therefore the 
values $ \tau = 0.1 $ and $ \tau = 0.2 $ lie in the strong adversary regime, 
while the values $ \tau = 0.5 $ and $ \tau = 0.7 $ belong to the weak adversary 
regime. 

In the weak adversary regime, the critical point is when 
\begin{equation}\label{eq:ber-critical}
n  \approx \Theta\left( \left( \frac{1}{\tau^2}-1 \right) \log 
\frac{1}{1-\eps/(\theta_0\tau)} \right)\,.
\end{equation} 
Before this critical point, the strip center $ s_n $ that the cross generalization 
gap $ g_n $ oscillates around is strictly increasing; it is strictly decreasing 
after the critical point and eventually vanished as $n\to \infty$. 
Note that when $ \tau\to 0 $, both terms ($ \left( \frac{1}{\tau^2}-1 \right) $ 
and $ \log 
\frac{1}{1-\eps/(\theta_0\tau)} $) in \eqref{eq:ber-critical} blow up and 
thereby the increasing stage elongates infinitely. The increasing and 
decreasing stages of the weak adversary regime are confirmed by the two curves 
$ \tau = 0.5 $ and $ \tau = 0.7 $ in \cref{fig:bernoulli}.

In the strong adversary regime, the strip center $ s_n $ displays a similar 
trend as the cross generalization gap in the Gaussian model; i.e., it is strictly 
increasing from the very beginning (see the two curves $ \tau=0.1 $ and $ \tau=0.2 $ in \cref{fig:bernoulli}). Recall that under the Bernoulli model, the strong/weak adversary regime is 
determined by the ratio $ \frac{\eps}{\theta_0\tau} $, while under the Gaussian 
model, it is determined by the ratio $ \frac{\eps}{\mu_0} $. Nevertheless, note that in the binary classification, $\theta_0 \tau$ is the mean (in one coordinate) of the positive class, just like $\mu_0$ in the Gaussian scenario. These two ratios are thus closely related.   

We would also like to remark that limits of $g_n$ in \cref{fig:bernoulli} follow the theoretical results outlined in \eqref{eq:ber-gn}. In particular, if we are in the weak adversary regime, the limit of $g_n$ always tends to 0. On the other hand, in the strong adversary regime, the limit is non-zero and proportional to $\tau$.

\subsection{Discussion}
One common observation from \cref{thm:gaussian} and \cref{thm:bernoulli} is that the duration of the increasing stage heavily depends on the ratio between $\eps$ and the coordinate-wise mean of the positive class (i.e. $\mu_0$ and $\theta_0 \tau$). Note that the mean can be interpreted as half the distance between the centers of positive and negative classes in the space of $x$. %
Thus, another way to view this result is that if the strength of the adversary is relatively large compared to the distance between classes, then we will have a long increasing stage.

One interesting implication of this can be seen in regression vs.\ classification tasks. Intuitively, one might look at a regression task as a classification task with infinitely many classes. Therefore, depending on the distribution that $x$ is sampled from, we could end up with a very small distance between class centers and thus we would expect a very long increasing stage.%

\section{Regression}\label{sec:regression}

\begin{figure*}[t]
	\begin{subfigure}{0.45\textwidth}
		\centering
		\includegraphics[width=\linewidth]{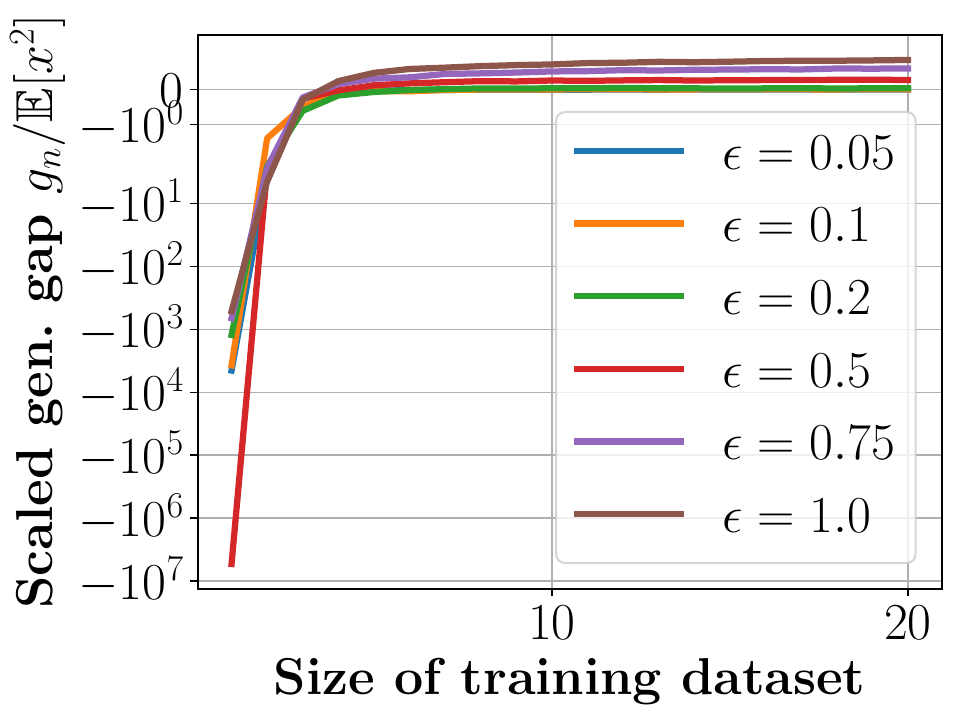}
		\caption{$ x\sim \cN(0,1) $, $ 1\le n\le 20 $.}
		\label{fig:reg-gaussian1}
	\end{subfigure}
	\hfill
	\begin{subfigure}{0.45\textwidth}
		\centering
		\includegraphics[width=\linewidth]{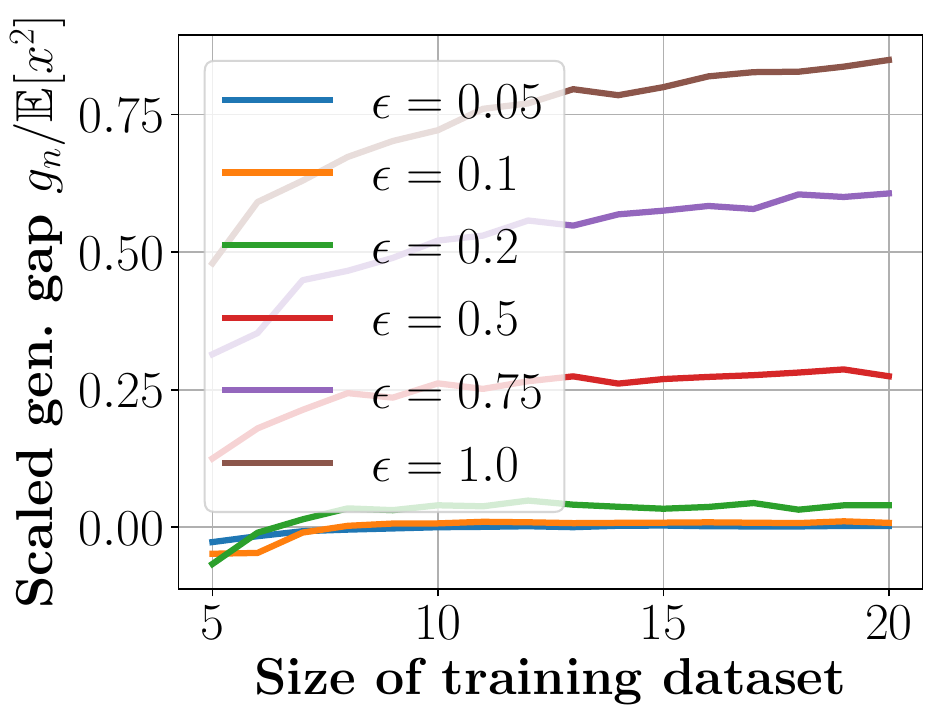}
		\caption{$ x\sim \cN(0,1) $, $ 5\le n\le 20 $.}
		\label{fig:reg-guassian5}
	\end{subfigure}
	
	\medskip
	
	\begin{subfigure}{0.45\textwidth}
		\centering
		\includegraphics[width=\linewidth]{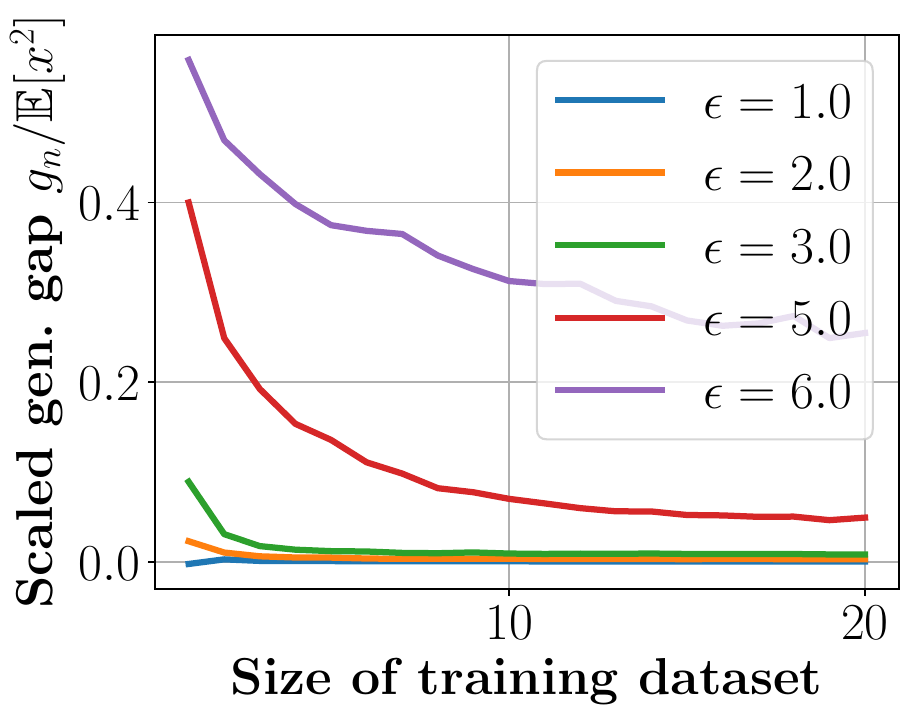}
		\caption{$ x\sim \pois(5) + 1 $, small $ \eps $.}
		\label{fig:poisson1}
	\end{subfigure}
	\hfill
	\begin{subfigure}{0.45\textwidth}
		\centering
		\includegraphics[width=\linewidth]{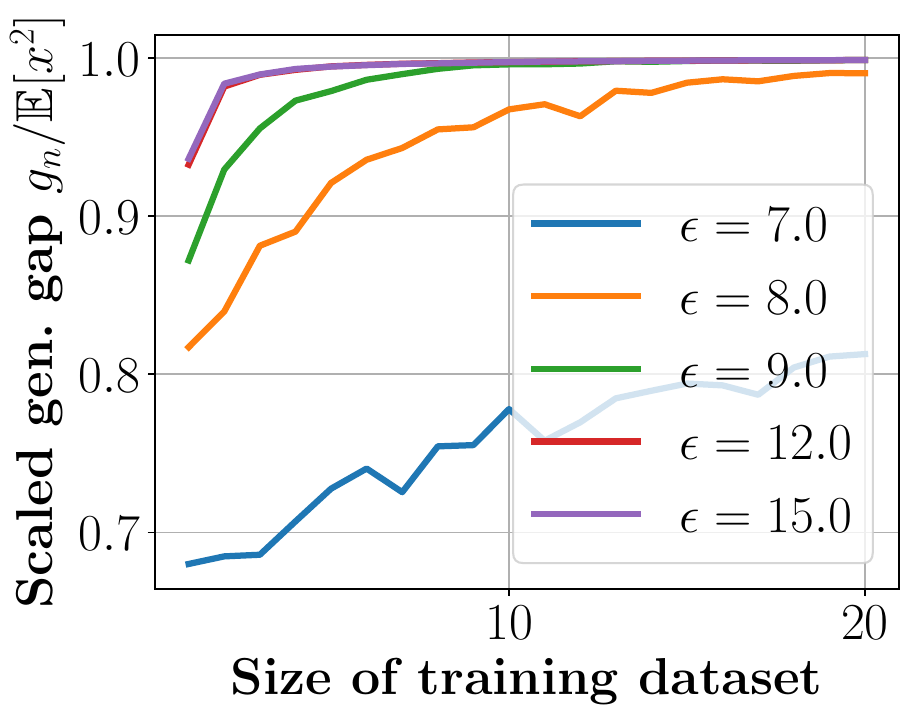}
		\caption{$ x\sim \pois(5) + 1 $, large $ \eps $.}
		\label{fig:poisson5}
	\end{subfigure}
	\caption{Scaled cross generalization gap $ g_n/\bE_{x\sim 
			P_X}[x^2] $ vs.\ the size of the training dataset 
		(denoted by $ n $). First two plots correspond to $x$ being sampled from the standard normal distribution $\cN(0,1)$ and last two plots correspond to $\pois(5) + 1$. Each curve in a plot represents a different choice of $\eps$.  }\label{fig:regression}
\end{figure*}

In this section, we explore the problem of linear regression, where we have each 
data point $ (x,y)\in \bR^d\times \bR $ and the linear model is represented by 
a vector $ w\in \bR^d $. The loss function is defined by $ \ell(x,y;w)  = 
(y-\langle w,x\rangle )^2 $. 

We assume the following data generation process. First, we sample $ x_i $ 
from some distribution $ P_X $. Given the fixed true model $ w^* $, we set $ y_i 
= 
\langle w^*,x_i\rangle + \delta $, where $ \delta\sim \cN(0,\sigma^2) $
is the 
Gaussian noise.
The parameter space $ \Theta $ is the entire $ \bR^d $. 

Given the training dataset $ \dtrain= \{(x_i,y_i)\}_{i=1}^n$, if we define $ X 
= [x_1,\dots,x_n]^\top $ and $ y = [y_1,\dots,y_n]^\top $,
the best standard model has a closed form \citep{graybill1961introduction}:
\begin{equation*}
\wstd = (X^\top X)^{-1} X^\top y\,.
\end{equation*}
 \cref{ob:regression-wstd-wrob} presents the form of the best robust model in 
 the linear regression problem.

\begin{obs}[\textbf{Proof in 
\cref{sec:proof-wrob-regression}}]\label{ob:regression-wstd-wrob}
	The best robust model in the linear regression problem is given by 
    \begin{equation*}
	\wrob_n = \argmin_{w\in \bR^d}\frac{1}{n} \sum_{i=1}^{n} \left(  |y_i - \langle w,x_i \rangle | + \eps \sum_{j=1}^d |w(j)|\right)^2 \,.
	\end{equation*}
\end{obs}

\cref{ob:regression-gn} gives the form of the gap in the linear regression problem setting.

\begin{obs}[\textbf{Proof in 
\cref{sec:proof-regression-gn}}]\label{ob:regression-gn}
In the linear regression problem, the cross generalization gap equals
	\begin{equation*}
	g_n = \| \wrob_n - w^* \|^2_{\bE_{x\sim P_X}[xx^\top]} - \| \wstd_n - w^* 
	\|^2_{\bE_{x\sim P_X}[xx^\top]}\,.
	\end{equation*}
\end{obs}

\cref{ob:regression-gn} shows that the cross generalization gap not only depends on 
the difference vectors $ (\wrob_n-w^*) $ and $ (\wstd_n-w^*) $ but also the 
matrix $ \bE_{x\sim P_X}[xx^\top] $. This matrix weights each dimension of the 
difference vectors and thereby influences the cross generalization gap.
 
To avoid the complication incurred by the different weightings of the matrix $ 
\bE_{x\sim P_X}[xx^\top] $ across the dimensions, we investigate two 
one-dimensional linear regression problems ($ 
d=1 $) with the data input $ x $ sampled from a standard normal distribution 
and a shifted Poisson distribution, respectively. 
To be specific, in the first study, we consider $ x $ sampled from the standard 
normal 
distribution $ \cN(0,1) $. In the second study, the data input $ x $ is drawn 
from $ \pois(5) + 1 $ (in order to avoid $ x = 0 $); in other words, $x-1$ obeys the $\pois(5)$ distribution. 
In both studies, we set the true model $ w^* = 1 $ and the noise obeys $ 
\delta\sim \cN(0,1) $ (\ie, $ \sigma^2=1 $). 
In light of \cref{ob:regression-gn}, we obtain that if the linear regression 
problem is one-dimensional, the cross generalization gap equals 
\begin{equation*}
g_n = \bE_{(x,y)\sim \mathcal{D}}\left( (\wrob_n - w^* )^2 - (\wstd_n - w^*)^2 \right) \bE_{x\sim 
	P_X}[x^2]\,.
\end{equation*}
Since $ g_n $ is proportional to $ \left( (\wrob_n - w^* )^2 - (\wstd_n - 
w^*)^2 \right) $ with $ \bE_{x\sim 
	P_X}[x^2] $ being a constant, we call $ g_n/\bE_{x\sim 
	P_X}[x^2] $ the \emph{scaled cross generalization gap} and plot it against the size 
	of the training dataset (denoted by $ n $) in \cref{fig:regression}. 
	
	\cref{fig:reg-gaussian1} shows the result for the first study with $ n $ 
	ranging from $ 1 $ to $ 20 $. For a clear presentation, 
	\cref{fig:reg-guassian5} provides a magnified plot for $ 5\le n\le 20 $.
	
	Our first observation is that in the Gaussian case, the cross generalization gap 
	$ g_n $ always expands with more data, even if $ \eps $ is as small as $ 
	0.05 $. This may be because if we sort $ n $ i.i.d.\ standard normal random 
	variables $ x_1,\dots,x_n $ in ascending order and obtain $ x_{\pi(1)} \le 
	x_{\pi(2)} \le  \cdots \le  x_{\pi(n)} $, the difference between two 
	consecutive numbers (\ie, $ x_{\pi(i+1)} - x_{\pi(i)} $) becomes smaller as 
	$ n $ becomes larger. As we discussed in \cref{sec:classification}, the 
	monotone trend of $ g_n $ is determined by the ratio of $ \eps $ to half the distance between the positive and negative classes. The ratio 
	is $ \frac{\eps}{\mu_0} $ in the Gaussian model and it is $ 
	\frac{\eps}{\theta_0\tau} $ in the Bernoulli model. The regression problem 
	may be viewed as a classification problem with infinitely many classes.  
	The difference between two consecutive numbers is the analog of the 
	distance between the means of difference classes. Since the difference 
	reduces as $ n $ becomes larger (points are more densely situated), the 
	ratio 
	increases and therefore we observe a wider cross generalization gap.  
	
	Our second observation regarding the Gaussian data is that the cross
	generalization gap is (very) negative at the initial stage. In particular, 
	when $ n=1 $, the gap $g_1$ is between $ -10^6 $ and $ -10^7 $. 
	The reason is 
	that when $ n=1 $, we have  \begin{equation*}
	\E\left[ (\wstd_1-w^*)^2 \right] 
	= \infty\,.
	\end{equation*} 
	Because of the robustness, $ \wrob $ is more stablized and therefore $ 
	\bE[\left( \wrob_1-w^* \right)^2] $ is finite. Since the cross generalization gap 
	$ g_1   $ is proportional to their difference  $ \bE[\left( 
	\wrob_1-w^* \right)^2] - \E\left[ (\wstd_1-w^*)^2 \right]  $, the gap $g_1$ is indeed 
	$ -\infty $. We present a proof of $g_1 = -\infty$ in \cref{thm:g1-minus-infty}.
	
	\begin{theorem}[\textbf{Proof in \cref{sec:proof-g1-minus-infty}}]\label{thm:g1-minus-infty}
		In the one-dimensional linear regression problem, if $ x_1\sim \cN(0,1) 
		$, $ \delta\sim \cN(0,1) $, and $ y_1 = w^* x+\delta $, the cross
		generalization gap $ g_1 $ with only one training data point is $ 
		-\infty $.
	\end{theorem}
	
	\cref{fig:poisson1} presents the result for the Poisson input with $ \eps  
	$ varying from $ 1.0 $ to $ 6.0 $. \cref{fig:poisson5} illustrates the 
	result corresponding to large $ \eps $ that ranges from $ 7.0 $ to $ 15.0 $.
	We see two different regimes in \cref{fig:poisson1} and \cref{fig:poisson5}. \cref{fig:poisson1} represents the weak adversary regime 
	where the cross generalization gap shrinks with more training data.
	  \cref{fig:poisson5} represents the strong adversary regime in which the 
	  gap expands with more training data. Furthermore, given the same size of 
	  the training dataset, the gap increases with $ \eps $. 
	  
	The result for the Poisson input is in sharp contrast to the Gaussian 
	input. It appears that for any small $ \eps $, the cross generalization gap will 
	increase with more data in the Gaussian setting, as the real line becomes 
	increasingly crowed with data points. In the Poisson setting, whilst the 
	Poisson distribution is infinitely supported as well, the minimum distance 
	between two different data points is one (recall that the Poisson 
	distribution is supported on natural numbers). A weak adversary with a 
	small $ \eps $ is unable to drive the cross generalization gap into an increasing 
	trend. Additionally, recalling that the mean of $ \pois(5) + 1 $ is $ 6 $, 
	the value $ \eps = 6 $ exactly separates the weak and strong adversary 
	regimes in these two figures. Note that all $ \eps $ values in \cref{fig:poisson1} 
	are  $ \le 6 $, while all those in \cref{fig:poisson5} are $ >6 $. 

Unlike the Gaussian setting for linear regression, we never observe a negative cross generalization gap, even if $n = 1$. This observation supports our theoretical finding, which is summarized in \cref{thm:g1-minus-infty-poisson}.

\begin{theorem}[\textbf{Proof in \cref{sec:proof-g1-minus-infty-poisson}}]\label{thm:g1-minus-infty-poisson}
		In the one-dimensional linear regression problem, if $ x_1\sim \pois(\lambda)+1 
		$, $ \delta\sim \cN(0,1) $, and $ y_1 = w^* x+\delta $ with $|w^*|\geq 1$, the cross generalization gap $ g_1 $ with only one training data point is non-negative, finite, and increases with $\eps$.
	\end{theorem}
	
\section{Conclusion}
In this paper, we study the cross generalization gap between adversarially robust models and standard models. We analyze two classification models (the Gaussian model and the Bernoulli model), and we also explore the linear regression model. We theoretically find that a larger training dataset won't necessarily close the cross generalization gap and may even expand it. In addition, for the two classification models, we prove that the cross generalization gap is always non-negative, which indicates that current adversarial training must sacrifice standard accuracy in exchange for robustness.

For the Gaussian classification model, we identify two regimes: the strong adversary regime and the weak adversary regime. In the strong adversary regime, the cross generalization gap monotonically expands towards some non-negative finite limit as more training data is used. On the other hand, in the weak adversary regime, there are two stages: an increasing stage where the gap increases with the training sample size, followed by a decreasing stage where the gap decreases towards some finite non-negative limit. Broadly speaking, the ratio between the strength of the adversary and the distance between classes determines which regime we will fall under.

In the Bernoulli model, we also prove the existence of the weak and strong adversary regimes. The primary difference is that the cross generalization gap is oscillating instead of monotone. However, we also show that these oscillating curves have strip centers that display very similar behavior to the Gaussian curves.

Our findings are further validated by a study of the linear regression model, which experimentally exhibits similar behavior and may indicate that our results hold for an even broader class of models. The ultimate goal of adversarial training is to learn models that are robust against adversarial attacks, but do not sacrifice any accuracy on unperturbed test sets. The primary implication of our work is that this trade-off is provably unavoidable for existing adversarial training frameworks.

\section*{Acknowledgements}
AK is partially supported by NSF (IIS-1845032), ONR (N00014-19-1-2406), and AFOSR (FA9550-18-1-0160). LC is supported by Google PhD Fellowship. We would like to thank Peter Bartlett, Hamed Hassani, Adel Javanmard, and Mohammad Mahmoody for their comments regarding the first version of the paper and thank Marko Mitrovic for his help in preparation of the paper.

\bibliographystyle{plainnat}
\bibliography{reference}
\clearpage
\onecolumn
\begin{appendices}
	\crefalias{section}{appsec}
	
	\section{Proof of \cref{thm:gaussian}}\label{sec:proof-gaussian}
	
	\begin{lemma}\label{lem:phi}
		Define the function $ \phi(x) = 2\Phi(x) - \Phi(x(1+\delta)) - 
		\Phi(x(1-\delta)) $, where $ \Phi $ denotes the CDF of the standard 
		normal distribution. We have
		\begin{enumerate}[label=(\alph*),nosep]
			\item If $ \delta >0 $, $ \lim_{x\to\infty} \phi(x) = H(\delta-1)  
			$, where $ H $ is the Heaviside step function.
			\item If $ \delta \in (0,1) $, there exists  \begin{equation*}
			 \sqrt{ \max\{ 
				\sfrac{3}{2},2\log\frac{1}{1-\delta} \}} < x_0 <
			\sqrt{K_0+2\log\frac{1}{1-\delta}}  
			\end{equation*} such that the 
				function 
			$ 
			\phi(x) $ is strictly increasing on $ (0, x_0) $ and strictly 
			decreasing on $ (x_0,\infty) $, where $ 
			K_0 >0 $ is a universal constant.
			\item If $ \delta\ge 1 $, the function $ 
			\phi(x) $ is strictly increasing on $ (0,\infty) $. 
		\end{enumerate}
	\end{lemma}
	\begin{proof}
		First, we compute 
		the 
		derivative of $ \phi(x) $ and obtain 
		\[
		\phi'(x) = \frac{1}{\sqrt{2 \pi }} e^{-\frac{1}{2} (\delta +1)^2 x^2} 
		\left(-\delta +(\delta 
		-1) e^{2 \delta  x^2}+2 e^{\frac{1}{2} \delta  (\delta +2) 
			x^2}-1\right)  \,.\] 
		If we define
		$ h(a) = -\delta +(\delta 
		-1) a^{2 \delta }+2 a^{\frac{1}{2} \delta  (\delta +2) 
		}-1 $ for $ a\ge 1 $, 
	we have $ \phi'(x) = \frac{1}{\sqrt{2 \pi }} 
		e^{-\frac{1}{2} (\delta 
			+1)^2 x^2}  h(e^{x^2}) $. It can be observed that $ h(1) = 0 $. 
		
		We first consider the case where $ \delta\ge 1 $. The derivative of $ 
		h(a) $ with respect to $ \delta $ is given by \[ 
		\frac{\partial h(a)}{\partial \delta} = a^{2 \delta }+2 \left((\delta 
		+1) a^{\frac{\delta ^2}{2}}+(\delta -1) a^{\delta }\right) a^{\delta } 
		\log (a)-1\,,
		 \]
		which is non-negative when $ a\ge 1 $ and $ \delta \ge 1 $. Therefore, 
		we deduce $ h(a)\ge h(a)\rvert_{\delta=1} = 4 a^{3/2} \log (a)+a^2-1 $. 
		Since the right-hand side is increasing in $ a $, we get $ h(a)\ge 
		h(a)\rvert_{\delta=1,a=1} = 0 $ and the equality is attained when $ a=0 
		$. In other words, $ h(a)>0 $ if $ a>1 $, which implies $ \phi'(x)>0 $ 
		for $ x>0 $. Therefore, the function $ \phi(x) $ is strictly increasing 
		on $ (0,\infty) $. 
		
		Next, we compute the limit $ \lim_{x\to \infty} 
		\phi(x) $. If $ \delta >1 $, when $ x $ goes to $ \infty $, we have $ 
		x(1+\delta) $ goes to $ \infty $ as well, while $ x(1-\delta) $ goes to 
		$ -\infty $. Recall that $ \Phi(x) $ is a CDF. Since $ 
		\lim_{x\to\infty} \Phi(x)=1 $ and $ \lim_{x\to -\infty} \Phi(x)=0 $, we 
		obtain that $ \lim_{x\to\infty} \phi(x) = 2-1-0 = 1 $. If $ \delta = 1 
		$, we have $ \phi(x) = 2\Phi(x) - \Phi(2x) - \Phi(0)  =2\Phi(x) - 
		\Phi(2x) -1/2 $. Therefore, we obtain $ \lim_{x\to\infty}\phi(x) = 
		2-1-1/2=1/2 $. 
		
		In the sequel, we assume $ \delta\in (0,1) $. 
		In this case, when $ x $ goes to $ \infty $, both $ x(1+\delta) $ and $ 
		x(1-\delta) $ go to $ \infty $. Therefore, we get $ 
		\lim_{x\to\infty}\phi(x) = 2-1-1 = 0 $.
		The 
		derivative of $ h(a) $ is given by 
		\[ h'(a) = \delta  a^{\delta -1} 
		\left((\delta +2) a^{\frac{\delta ^2}{2}}+2 (\delta -1) a^{\delta 
		}\right) 
		\,.\]
		
		 If $ \delta\in (0,1) $, the function $ h'(a) $ is positive on $ (1, 
		a_0) $ and negative on $ (a_0,\infty) $, where $ 
		a_0=\left( \frac{2+\delta}{2(1-\delta)} 
		\right)^{\frac{1}{\delta-\delta^2/2}} $. 
		Therefore, if $ \delta>0 $, the function $ h(a) $ is strictly 
		increasing on 
		$ (1,a_0) $ and strictly decreasing on $ (a_0,\infty) $. Since $ h(1)=0 
		$ 
		and $ \lim_{a\to\infty}h(a) = -\infty $, we deduce that $ h(a) $ has a 
		unique root $ a_1 $ on $ (1,\infty) $ and $ a_0<a_1 $. 
		
		We claim that $ 
		a_0(\delta)= \left( \frac{2+\delta}{2(1-\delta)} 
		\right)^{\frac{1}{\delta-\delta^2/2}} $ is increasing with respect 
		to $ \delta\in (0,1) $. We define $ f(\delta) = \log a_0(\delta) = 
		\frac{\log \left(\frac{\delta +2}{2 (1-\delta )}\right)}{\delta 
			-\frac{\delta ^2}{2}} $ and $ f_1(\delta) = 4 
		\left(\delta ^3-3 \delta +2\right) \log \left(\frac{\delta +2}{2-2 
		\delta 
		}\right)+6 (\delta -2) \delta $. The derivative of $ f(\delta) $ is 
		given by $ 
		f'(\delta) = \frac{4 
			\left(\delta ^3-3 \delta +2\right) \log \left(\frac{\delta +2}{2-2 
			\delta 
			}\right)+6 (\delta -2) \delta }{(\delta -2)^2 \delta ^2 
			\left(\delta 
			^2+\delta -2\right)} = \frac{f_1(\delta)}{(\delta -2)^2 \delta ^2 
			\left(\delta 
			^2+\delta -2\right)} $. The derivative of $ f_1(\delta) $ is $ 
			f'_1(\delta) 
		= 12 \left(\delta ^2-1\right) \log \left(\frac{\delta +2}{2-2 \delta 
		}\right) $. Since $ \delta>0 $, we have $ \delta+2 > 2-2\delta $ and 
		thus $ 
		\log \left(\frac{\delta +2}{2-2 \delta 
		}\right) > 0 $. Therefore $ f'_1(\delta) < 0 $ on $ (0,1) $. As a 
		result, $ 
		f_1(\delta) $ is decreasing on $ (0,1) $ and thus for $ \delta\in (0,1) 
		$, we 
		have $ f_1(\delta) < f_1(0) = 0 $. Since $ (\delta -2)^2 \delta ^2 
		\left(\delta 
		^2+\delta -2\right) < 0 $ holds for $\delta\in  (0,1) $, the derivative 
		$ 
		f'(\delta)>0 $ on $ (0,1) $. Therefore the function $ f(\delta) $ is 
		increasing on $ (0,1) $ and for $ \forall \delta\in (0,1) $, we have $ 
		f(\delta) \ge \lim_{\delta\to 0+} f(\delta) = \frac{3}{2} $. Thus $ 
		a_0(\delta) $ is increasing on $ (0,1) $ and $ a_0(\delta)\ge e^{3/2} $.
		
		Since $ a_0<a_1 $, we have $ a_1>e^{3/2} $. Next, we show that $ a_1 > 
		\frac{1}{(1-\delta)^2} $. 
		Since $a_1 > e^{3/2} $, it suffices to show $a_1 
		> \frac{1}{(1 - \delta)^2 }$ for $\delta > 
		1-e^{-\sfrac{3}{4}}>\sfrac{1}{2}$. As a 
		result, in 
		what follows, we assume $  \sfrac{1}{2} < \delta < 1 $.
		
		First, since $  \sfrac{1}{2} < \delta < 1 $, the following inequality 
		holds \begin{equation}\label{eq:le_onehalf}
		(1-\delta)^{\delta(2+\delta)} < (1-\delta)^{\sfrac{5}{4}} < 
		\frac{1}{2}\,.
		\end{equation}
		Therefore, we deduce \begin{equation*}
		 (1-\delta)^{(1-\delta)^2} + (1-\delta)^{\delta(2+\delta)} (1+\delta)
		< 1 + (1-\delta)^{\delta(2+\delta)} (1+\delta) < 1 + 
		2(1-\delta)^{\delta(2+\delta)} < 1+2\cdot \frac{1}{2} = 2\,,
		\end{equation*}
		where we use \eqref{eq:le_onehalf} in the last inequality. Since 
		\begin{equation*}
		(1-\delta)^{\delta(2+\delta)}\left( 1+\delta+(1-\delta)^{1-4\delta} 
		\right) = (1-\delta)^{(1-\delta)^2} + (1-\delta)^{\delta(2+\delta)} 
		(1+\delta) < 2\,,
		\end{equation*}
		we have \begin{equation*}
		2(1-\delta)^{-\delta(2+\delta)} > 1+\delta+(1-\delta)^{1-4\delta}\,.
		\end{equation*}
		
		We are in a position to evaluate $ h\left( \frac{1}{(1-\delta)^2}
		\right) $: \begin{equation*}
		h\left( \frac{1}{(1-\delta)^2}
		\right) = 2 (1-\delta )^{-\delta  (\delta +2)}-(1-\delta )^{1-4 \delta 
		}-(1 + \delta ) > 0\,.
		\end{equation*}

	 Therefore we get 
		$h(a_1) = h(\frac{1}{(1-\delta)^2}) > 0 $. Recall that $ a_1 $ is the 
		unique root of $ h(a) $ and that $ h(a) >0 $ if $ a\in (0,a_1) $ while 
		$ h(a)<0 $ if $ a > a_1 $. Therefore we have $ a_1 >  
		\frac{1}{(1-\delta)^2} $ if $ \delta > \sfrac{1}{2} $. As a 
		consequence, we 
		conclude $a_1 > \max \{ e^{3/2} , \frac{1}{(1-\delta)^2} 
		\}$. 
		
		In the final part of the proof, we derive an upper bound for $ a_1 $. 
		We consider the function \begin{equation*}
		h_1(\delta,K)  = -(\delta +1) (1-\delta )^{\delta  (\delta +2)}+2 
		K^{\frac{\delta ^2}{2}+\delta }-(1-\delta )^{(1-\delta )^2} K^{2 
		\delta }\,.
		\end{equation*}
		First, we claim that there exists $ K_1>0 $ such that for all $ K>K_1 
		$, $ h_1(\delta,K) $ is decreasing in $ K $ for every given $ \delta\in 
		(0,1) $. To see this, we compute its derivative with respect to $ K $: 
		\begin{equation*}
		\frac{\partial h_1}{\partial K} = \delta  K^{\delta -1} \left((\delta 
		+2) K^{\frac{\delta ^2}{2}}-2 (1-\delta )^{(\delta -1)^2} K^{\delta 
		}\right)\,.
		\end{equation*}
		The derivative is negative if $ K > h_2(\delta)\triangleq (\delta -1)^2 
		\left(\frac{6}{\delta 
		+2}-2\right)^{\frac{2}{(\delta -2) \delta }} $. The function $ 
		h_2(\delta) $ is continuous on $ (0,1) $. Since $ \lim_{\delta\to 0+} 
		h_2(\delta) = e^{3/2} $ and $ \lim_{\delta\to 1-} h_2(\delta) = 
		\frac{9}{4} $, $ h_2(\delta) $ is bounded on $ (0,1) $. Therefore, if 
		we set $ K_1 = \sup_{\delta\in (0,1)} h_2(\delta) $, we have $ 
		\frac{\partial h_1}{\partial K} < 0 $ and thereby $ h_1(\delta,K) $ is 
		decreasing for all $ K > K_1 $.
		
		In the second step, we expand $ h_1 $ with respect to $ \delta $ around 
		$ 0 $ and with the 
		Peano form of the remainder
		\begin{equation*}
		h_1(\delta,K) = \delta ^2 \left(3 \log (K)-\log ^2(K) + 
		r_K(\delta)\right)\,,
		\end{equation*}
		where $ \lim_{\delta\to 0+} r_K(\delta) = 0 $.
		Since $ 3 \log (K)-\log ^2(K) < 0 $ for all $ K\ge 21 $, there exists $ 
		\delta_0(K) > 0 $ such that for all $ \delta \in (0, \delta_0(K)) $, we 
		have 
		$ h_1(\delta, K) < 0  $; in particular, for $ K_2 = \max\{21, K_1\} $, 
		there exists $ 
		\delta_0(K_2) > 0 $ such that for all $ \delta \in (0, \delta_0(K_2)) 
		$, we 
		have 
		$ h_1(\delta, K_2) < 0  $. We use the shorthand $ \delta_2 \in (0,1) $ 
		to 
		denote $ \delta_0(K_2) $. Since for every given $ \delta\in (0,1) $, $ 
		h_1(\delta,K) $ is decreasing for all $ K> K_2 \ge K_1 $, we have
		\begin{equation}\label{eq:small_delta}
		h_1(\delta,K)<0,\quad \forall \delta\in (0,\delta_2) , K>K_2 \,.
		\end{equation}
		
		Let $ h_3(\delta) = (1-\delta )^2 \log (1-\delta ) $. Its derivative is 
		$ h'_3(\delta) = -(1-\delta ) (2 \log (1-\delta )+1) $, which is 
		positive if $ \delta > 1-e^{-\sfrac{1}{2}} $ and negative if $ \delta < 
		1-e^{-\sfrac{1}{2}} $. Therefore, for all $ \delta\in (0,1) $, we have 
		$ h_3(\delta) \ge h_3(1-e^{-\sfrac{1}{2}}) = -\frac{1}{2 e} $. As a 
		consequence, for all $ \delta\in (0,1) $, we have 
		\begin{equation}\label{eq:ge_onehalf}
			(1-\delta)^{(1-\delta)^2} \ge \exp(-\frac{1}{2 e}) > \sfrac{1}{2}\,.
		\end{equation} 
		
		If $ K > K_3 =  4^{\frac{1}{\delta_2-\delta_2^2/2}} $, we have for all 
		$ \delta \in [\delta_2,1) $ it holds that $ K >   
		4^{\frac{1}{\delta_2-\delta_2^2/2}} \ge 4^{\frac{1}{\delta-\delta^2/2}} 
		$ and therefore
		\begin{equation}\label{eq:2k}
		2K^{\delta^2/2+\delta} - \frac{1}{2} K^{2\delta}<0\,.
		\end{equation}
		Hence, for $ \delta\in [\delta_2,1) $ and $ K>K_3 $, the follow 
		inequalities hold 
		\begin{equation}\label{eq:large_delta}
		\begin{split}
		h_1(\delta,K)  ={}& -(\delta +1) (1-\delta )^{\delta  (\delta +2)}+2 
		K^{\frac{\delta ^2}{2}+\delta }-(1-\delta )^{(1-\delta )^2} K^{2 
			\delta }\\
		<{} & 2 
		K^{\frac{\delta ^2}{2}+\delta }-(1-\delta )^{(1-\delta )^2} K^{2 
			\delta }\\
		<{} & 2 
		K^{\frac{\delta ^2}{2}+\delta }-\frac{1}{2} K^{2 
			\delta } < 0\,,
		\end{split}
		\end{equation}
		where the first inequality holds because $ (\delta +1) (1-\delta 
		)^{\delta  (\delta +2)} > 0 $, the second inequality holds due to 
		\eqref{eq:ge_onehalf}, and the final inequality holds because of 
		\eqref{eq:2k}.
		
		Combining \eqref{eq:small_delta} and \eqref{eq:large_delta}, we deduce 
		that for  $ K_4 = \max\{K_2,K_3\} + 1 $ and all $ \delta\in (0,1) $, we 
		have $ h_1(\delta,K_4) < 0 $.
		
		Notice that \begin{equation*}
		h\left( \frac{K_4}{(1-\delta)^2} \right) = 
		\frac{h_1(\delta,K_4)}{(1-\delta )^{\delta  (\delta +2)}} < 0\,.
		\end{equation*}
		Therefore, we have $ a_1 < \frac{K_4}{(1-\delta)^2} $.
		
		Since $ \phi'(x) = \frac{1}{\sqrt{2 \pi }} 
		e^{-\frac{1}{2} (\delta 
			+1)^2 x^2}  h(e^{x^2}) $, we have $ e^{x_0^2} = a_1 $ and the 
			function $ \phi(x) $ is strictly increasing on $ (0,x_0) $ and 
			strictly decreasing on $ (x_0,\infty) $. Recalling $ 
			\frac{K_4}{(1-\delta)^2} > a_1 > \max \{ 
			e^{3/2} , \frac{1}{(1-\delta)^2} 
			\}$ and setting $ K_0 = \log K_4 >0  $, we have $ 
			\sqrt{K_0+2\log\frac{1}{1-\delta}} > x_0 = 
			\sqrt{\log(a_1)} > \sqrt{ \max\{ 
			\sfrac{3}{2},2\log\frac{1}{1-\delta} \}} $.

	\end{proof}

	\begin{lemma}\label{lem:wstd_wrob_gap}
		Given $ \eps>0 $ and training data $ 
		\{(x_i,y_i)\}_{i=1}^n \subseteq \bR^d\times \{\pm 1\}  $ with $ n $ 
		data points, if we define the 
		standard and robust 
		classifier by \begin{align*}
		\wstd_n ={}& \argmax_{\|w\|_\infty\le W} \sum_{i=1}^n 
		y_i\langle w,x_i\rangle\,, \\
		\wrob_n = {}& \argmax_{\|w\|_\infty\le W} \sum_{i=1}^n 
		\min_{\tilde{x}_i\in B^\infty_{x_i}(\eps)}
		y_i\langle w,\tilde{x_i}\rangle\,,
		\end{align*}
		we have $ \wstd - \wrob = W[\sign(u) - \sign(u-\eps\sign(u))] $, where 
		$ u 
		= \frac{1}{n}\sum_{i=1}^n y_i x_i $.
	\end{lemma}
	\begin{proof}
		The first step is to compute the inner minimization in the expression 
		of the robust classifier. If $ y_i=1 $, the minimizer of $ 
		\min_{\tilde{x}_i\in 
			B^\infty_{x_i}(\eps)}
		y_i\langle w,\tilde{x_i}\rangle $ is $ x_i-\eps \sign(w) $. If $ y_i=-1 
		$, its minimizer is $ x_i + \eps\sign(w) $. Therefore, in both cases, 
		the minimizer is $ x_i - y_i\eps\sign(w) $ and \[ 
		\min_{\tilde{x}_i\in 
			B^\infty_{x_i}(\eps)}
		y_i\langle w,\tilde{x_i}\rangle = y_i\langle w, x_i - 
		y_i\eps\sign(w)\rangle = y_i(\langle x_i,w\rangle -y_i\eps \|w\|_1) = 
		y_i\langle x_i,w\rangle - \eps\|w\|_1\,.
		\]
		Thus we have \begin{equation}\label{eq:equiv-rob}
		\sum_{i=1}^n 
		\min_{\tilde{x}_i\in B^\infty_{x_i}(\eps)}
		y_i\langle w,\tilde{x_i}\rangle  = \sum_{i=1}^n (y_i\langle 
		x_i,w\rangle - \eps\|w\|_1) = n(\langle u,w\rangle - \eps \|w\|_1) = n 
		\sum_{\substack{j\in [d]: \\w(j)\ne 0}} (u(j)w(j) - \eps |w(j)|)\,,
		\end{equation}
		where $ u = \frac{1}{n} \sum_{i=1}^n y_i x_i $. By the definition of 
		the robust 
		classifier, $ \wrob_n $ is a maximizer of \eqref{eq:equiv-rob}. We 
		only consider $ j\in [d] $ such that $ w(j)\ne 0 $. If $ u(j)\ne 0 $, 
		we have $ \sign(\wrob_n(j)) = 
		\sign(u(j)) 
		$; otherwise, we can always flip the sign of $ \wrob(j) $ and make 
		\eqref{eq:equiv-rob} larger (note that the first term $ \langle 
		u,w\rangle $ will increase and the second term $ - \eps \|w\|_1 $ 
		remains unchanged). If $ u(j)=0 $, to maximize the second term $ - 
		\eps \|w\|_1 $, $ \wrob_n(j) $ has to be zero. Therefore, we conclude 
		that $ \sign(\wrob_n) = \sign(u) $ and obtain
		\[ 
		\wrob_n = \argmax_{\|w\|_\infty\le W} \langle u-\eps \sign(u),w\rangle
		= W\sign(u-\eps \sign(u))\,.
		\]
		The standard classifier equals \[ 
		\wstd_n = \argmax_{\|w\|_\infty\le W} \langle w,\sum_{i=1}^n 
		y_ix_i\rangle = W\sign(u)\,.
		\]
		Therefore, we obtain that $ \wstd - \wrob = W[\sign(u) - 
		\sign(u-\eps\sign(u))] $.
	\end{proof}
	
	\begin{proof}[Proof of \cref{thm:gaussian}]
		
		  Given the training data $ \{(x_i,y_i)\}_{i=1}^n $, 
		  \cref{lem:wstd_wrob_gap} implies that the generalization 
		  gap is given by \begin{align*}
		  \bE_{(x,y)\sim \cD}[y\langle \wstd-\wrob, x\rangle ] ={}& 
		  \frac{\langle 
		  \wstd-\wrob, \mu\rangle - \langle \wstd-\wrob, -\mu\rangle}{2} = 
		  \langle 
	  \wstd-\wrob, \mu\rangle\\
	  ={}& W\sum_{j\in [d]} \mu(j)[ \sign(u(j)) - \sign(u(j)-\eps\sign(u(j))) ]
	  \,,
		   \end{align*}
		where $ u 
		= \frac{1}{n}\sum_{i=1}^n y_i x_i $.
	Note that $ u $ is distributed as $ \cN(\mu,\frac{1}{n}\Sigma) $. We have $ 
	u(j)\sim \cN(\mu(j),\frac{\sigma(j)^2}{n} ) $. Therefore, we deduce 
	\begin{align*}
	& \bE_{u(j)\sim \cN(\mu(j), \frac{\sigma(j)^2}{n})} [\sign(u(j)) - 
	\sign(u(j)-\eps\sign(u(j)))]\\
	 ={}& 2[\Pr[0<u(j)<\eps]-\Pr[-\eps<u(j)<0]]\\
	 ={}& 2[\Phi(\frac{\sqrt{n}}{\sigma(j)}(\eps-\mu(j))) - 
	 \Phi(-\frac{\sqrt{n}}{\sigma(j)}\mu(j)) - 
	 \Phi(-\frac{\sqrt{n}}{\sigma(j)}\mu(j)) + 
	 \Phi(-\frac{\sqrt{n}}{\sigma(j)}(\eps+\mu(j))) ]\\
	 ={}& 2[2\Phi(\frac{\sqrt{n}}{\sigma(j)}\mu(j))- 
	 \Phi(\frac{\sqrt{n}}{\sigma(j)}(\mu(j)+\eps)) - 
	 \Phi(\frac{\sqrt{n}}{\sigma(j)}(\mu(j)-\eps))]\,,
	 \end{align*}
	 where $ \Phi $ denotes the CDF of the standard normal distribution. We 
	 are in a position to compute $ g_n $:
	 \begin{equation}\label{eq:gn}
	 	\begin{split}
	 	g_n ={}& 2W\sum_{j\in [d]:\mu(j)\ne 0} \mu(j) 
	 	[2\Phi(\frac{\sqrt{n}}{\sigma(j)}\mu(j))- 
	 	\Phi(\frac{\sqrt{n}}{\sigma(j)}(\mu(j)+\eps)) - 
	 	\Phi(\frac{\sqrt{n}}{\sigma(j)}(\mu(j)-\eps))]\\
	 	={}& 2W\sum_{j\in [d]:\mu(j)\ne 0} \mu(j) 
	 	[2\Phi(\frac{\sqrt{n}}{\sigma(j)}\mu(j))- 
	 	\Phi(\frac{\sqrt{n}}{\sigma(j)}\mu(j)(1+\eps')) - 
	 	\Phi(\frac{\sqrt{n}}{\sigma(j)}\mu(j)(1-\eps'))]
	 	\,,
	 	\end{split}
	 \end{equation}
	  where $ \eps' = \frac{\eps}{\mu(j)} $. The second derivative of $ \Phi(x) 
	  $ is $ \Phi''(x) = -\frac{e^{-\frac{x^2}{2}} x}{\sqrt{2 \pi }} $ and it 
	  is non-positive if $ x\ge 0 $. This implies the concavity of $ \Phi $ on 
	  $ [0,\infty) $. By Jensen's inequality, we have \[ 
	  \Phi(\frac{\sqrt{n}}{\sigma(j)}\mu(j))- 
	  \frac{1}{2}(\Phi(\frac{\sqrt{n}}{\sigma(j)}\mu(j)(1+\eps')) + 
	  \Phi(\frac{\sqrt{n}}{\sigma(j)}\mu(j)(1-\eps'))) \ge 0\,.
	   \]
	   Therefore $ g_n\ge 0 $.

When $ n $ goes to $ \infty $, if $ mu(j)>0 $, $ 
\frac{\sqrt{n}}{\sigma(j)}\mu(j) $ goes to $ \infty $ as well. By 
\cref{lem:phi}, we get \[ 
\lim_{n\to\infty} g_n = 2W \sum_{j\in [d]:\mu(j)>0} \mu(j) H\left( 
\frac{\eps}{\mu(j)}-1 \right)\,.
 \]

If $ \eps < \min_{j\in [d]:\mu(j)>0} \mu(j) $, we have for all $ j\in [d] $ 
such that $ \mu(j)>0 $, it holds that $ \eps <\mu(j) $.
  Recalling \eqref{eq:gn} and by \cref{lem:phi}, we 
 	deduce 
 that $ g_n $ is strictly increasing if $$ \frac{\sqrt{n}}{\sigma(j)}\mu(j) < 
\sqrt{ \max\left\{ 
	\frac{3}{2},2\log\frac{1}{1-\eps/\mu(j)} \right\}} $$ for $\forall j\in [d] 
	$ 
	such that 
	$ \mu(j)>0 $. In other 
 words, $ g_n $ is strictly increasing when \[
 n <  
 \min_{j\in [d]:\mu(j)>0} \max\left\{ 
 \frac{3}{2} ,2\log\frac{1}{1-\eps/\mu(j)} \right\} \left(  
 \frac{\sigma(j)}{\mu(j)} 
 \right)^2 \,.
 \] 
 Since $ \phi(x) 
 $ is strictly decreasing when $ x $ is sufficiently large (\ie, $ x \ge 
 \sqrt{K_0 + 2\log \frac{1}{1-\delta}} $, where $ K_0 $ is a universal 
 constant), we have $ g_n $ is strictly decreasing if \begin{equation*}
 \frac{\sqrt{n}}{\sigma(j)} \mu(j) \ge \sqrt{K_0 + 2\log 
 \frac{1}{1-\eps/\mu(j)}}
 \end{equation*}
 for $ \forall j\in [d] $ such that $ \mu(j) > 0 $. In other words, $ g_n $ is 
 strictly increasing when \begin{equation*}
 n \ge \max_{j\in [d]: \mu(j)>0}\left( K_0 + 2\log \frac{1}{1-\eps/\mu(j)} 
 \right)\left( \frac{\sigma(j)}{\mu(j)} \right)^2\,.
 \end{equation*}
 
 If $ \eps > \| \mu \|_\infty $, we have for all $ j\in [d] $ such that $ 
 \mu(j)>0 $, it holds that $ \eps > \mu(j) $. \cref{lem:phi} gives that $ g_n $ 
 is strictly increasing for all $ n\ge 1 $. 
 
\end{proof}

\section{Proof of \cref{thm:bernoulli}}\label{sec:proof-bernoulli}
\begin{lemma}[Berry–Esseen~\citep{berry1941accuracy}]\label{lem:bet}
	There exists a positive constant $ C_0 $ such that if $ X_1,X_2,\dots $ 
	are i.i.d.\ random variables with $ \bE[X_1]=0 $, $ \bE[X_1^2] = 
	\sigma^2 > 0 $, and $ \bE[|X_1|^3]=\rho<\infty $, and if we define $ 
	Y_n = \frac{1}{n}\sum_{i=1}^n X_i $ and denote the CDF of $ 
	\frac{Y_n\sqrt{n}}{\sigma} $ by $ F_n $, then for all $ x $ and $ n $, 
	\[ 
	|F_n(x)-\Phi(x)| \le \frac{C_0\rho}{\sigma^3 \sqrt{n}}\,,
	\]
	where $ \Phi(x) $ is the CDF of the standard normal distribution. 
\end{lemma}
\citet{shevtsova2014absolute} established the upper bound $ C_0\le 0.4748 
$. In the Bernoulli case where the cardinality of the support of $ X_1 $ is 
$ 2 $, \citet{schulz2016} showed that $ C_0 \le 
\frac{\sqrt{10}+3}{6\sqrt{2\pi}}\approx 0.4097 $. 
\lc{revise}
\begin{lemma}\label{lem:binom_dist}
	If $ X\sim \bin(n,p) $, $p>\sfrac{1}{2} $, $ \delta >0 $ and $x = \sfrac{n}{2}$, 
	we 
	have $ \Pr[X\in (x,x+\delta)] \ge \Pr[X\in (x-\delta, x)] $.
\end{lemma}
\begin{proof}

We first note that the two intervals $(x,x+\delta)$ and $(x-\delta, x)$ are symmetric about $x=\sfrac{n}{2}$. Thus if $\exists l \in (x-\delta, x)$, such that $\Pr(X=l) > 0$, \ie, $l$ is an integer, then there exists a positive number $k$ such that $l=\sfrac{n}{2}-k$, and $\sfrac{n}{2}+k$ is an integer and falls on $(x,x+\delta)$. Actually, this is a bijection: there is a set $K$ of positive number such that $\{\sfrac{n}{2}-k: k \in K\} = (x-\delta,x)\cap \mathbb{Z}$, and $\{\sfrac{n}{2}+k: k \in K\} = (x,x+\delta)\cap \mathbb{Z}$.

So we have
\begin{align*}
\Pr[X\in (x-\delta, x)] &{}= \Pr[X \in \{\sfrac{n}{2}-k: k \in K\}]  \\
&{}= \sum_{k \in K} \binom{n}{\sfrac{n}{2}-k}p^{\sfrac{n}{2}-k}(1-p)^{\sfrac{n}{2}+k} \\
&{}\leq \sum_{k \in K} \binom{n}{\sfrac{n}{2}-k}p^{\sfrac{n}{2}+k}(1-p)^{\sfrac{n}{2}-k} \\
&{} = \Pr[X \in \{\sfrac{n}{2}+k: k \in K\}] \\
&{} = \Pr[X\in (x,x+\delta)],
\end{align*}
where the inequality holds because $p > \sfrac{1}{2}$.
\end{proof}

\begin{proof}[Proof of \cref{thm:bernoulli}]
	For $ i\in [n] $ and $ j \in [d] $, let $ B_i(j) 
	\stackrel{\textnormal{i.i.d.}}{\sim} \ber(\frac{1+\tau}{2}) $. We have $ 
	x_i(j) = (2B_i(j)-1)y_i\theta(j) $ and $ \bE[x_i\mid y_i] = y_i\theta\tau 
	$. If we define 
	$ u(j) = \frac{1}{n}\sum_{i=1}^{n} y_ix_i(j) $, we have 
	 \[ u(j) = 
	\frac{1}{n}\sum_{i=1}^{n} (2B_i(j)-1)\theta(j)\,.\]
	Given the training data $ \{(x_i,y_i)\}_{i=1}^n $, 
	\cref{lem:wstd_wrob_gap} implies that the generalization 
	gap is given by \begin{align*}
	\bE_{(x,y)\sim \cD}[y\langle \wstd-\wrob, x\rangle ] ={}& 
	\frac{\langle 
		\wstd-\wrob, \theta\tau\rangle - \langle \wstd-\wrob, 
		-\theta\tau\rangle}{2} = 
	\tau\langle 
	\wstd-\wrob, \theta\rangle\\
	={}& W\tau \sum_{j\in [d]} \theta(j)[ \sign(u(j)) - 
	\sign(u(j)-\eps\sign(u(j))) ]
	\,.
	\end{align*}
For $ j\in [d] $ such that $ \theta(j)>0 $,
	taking the expectation over $ u(j) $, we deduce \lc{doublecheck the 
	situation where $ \sign=0 $}
\begin{align*}
	& \bE_{u(j)} [\sign(u(j)) - 
	\sign(u(j)-\eps\sign(u(j)))]\\
	={}& 2[\Pr[0<u(j)<\eps]-\Pr[-\eps<u(j)<0]]\\
	={}& 2\left( \Pr\left [\frac{n}{2}< \sum_{i=1}^{n} B_i(j) < 
	\frac{n}{2}\left( \frac{\eps}{n}+1 \right) \right ] - 
	\Pr\left [ \frac{n}{2}\left( 1-\frac{\eps}{n} \right) <  
	\sum_{i=1}^{n} B_i(j) < \frac{n}{2}\right  ]
	 \right)\,.
\end{align*}
Since $ 
x_i(j) = (2B_i(j)-1)y_i\theta(j) $, the sum $ \sum_{i=1}^{n}B_i(j) $ obeys the 
distribution $ \bin(n, \frac{1+\tau}{2}) $, where $ \frac{1+\tau}{2} > 
\frac{n}{2} $. \cref{lem:binom_dist} implies that \[ 
\Pr\left [\frac{n}{2}< 
\sum_{i=1}^{n} B_i(j) < 
\frac{n}{2}\left( \frac{\eps}{n}+1 \right) \right ] \ge 
\Pr\left [ \frac{n}{2}\left( 1-\frac{\eps}{n} \right) <  
\sum_{i=1}^{n} B_i(j) < \frac{n}{2}\right  ]\,.
 \] 
 Therefore, we have $ \bE_{u(j)} [\sign(u(j)) - 
 \sign(u(j)-\eps\sign(u(j)))]\ge 0 $. Since \[
 g_n = W\tau \sum_{j\in [d]} 
 \theta(j) \bE_{u(j)} [\sign(u(j)) - 
 \sign(u(j)-\eps\sign(u(j)))]
 \]
 and we assume $ \theta(j)\ge 0 $, we deduce $ g_n\ge 0 $. 
 
 Next, we show that $ g_n $ is contained in a strip centered at $ s_n $ and 
 with width $ O\left( \frac{1}{\sqrt{n}} \right) $. We have
	\begin{align*}
	& \bE_{u(j)} [\sign(u(j)) - 
	\sign(u(j)-\eps\sign(u(j)))]\\
	={}& 2[\Pr[0<u(j)<\eps]-\Pr[-\eps<u(j)<0]]\\
	={}& 2[\Pr[ -\frac{\tau\sqrt{n}}{\sqrt{1-\tau^2}}
	\le \frac{1}{\sqrt{n}} 
	\sum_{i=1}^{n} 
	X_i \le
	\frac{(\frac{\eps}{\theta(j)}-\tau)\sqrt{n}}{\sqrt{1-\tau^2}} 
	]-
	\Pr[-\frac{(\tau+\frac{\eps}{\theta(j)})\sqrt{n}}{\sqrt{1-\tau^2}}
	 \le 
	\frac{1}{\sqrt{n}}\sum_{i=1}^{n} X_i \le 
	-\frac{\tau\sqrt{n}}{\sqrt{1-\tau^2}}]]\\
	={}& 2[F_n(\frac{(\frac{\eps}{\theta(j)}-\tau)\sqrt{n}}{\sqrt{1-\tau^2}}) 
	+ F_n(-\frac{(\tau+\frac{\eps}{\theta(j)})\sqrt{n}}{\sqrt{1-\tau^2}})
	- 2F_n(-\frac{\tau\sqrt{n}}{\sqrt{1-\tau^2}})
	  ]
	\end{align*}
	where $ X_i = \frac{2(B_i(j)-\frac{1+\tau}{2})}{\sqrt{(1-\tau^2)}} $ and $ 
	F_n $ is the CDF of $ \frac{1}{\sqrt{n}}\sum_{i=1}^{n}X_n $.
	The third absolute moment of $ X_1 $ is $ \bE[|X_1|^3] = \frac{\tau 
	^2+1}{\sqrt{1-\tau ^2}} $. By \cref{lem:bet}, we get 
\begin{align*}
&\left|
\bE_{u(j)} [\sign(u(j)) - 
\sign(u(j)-\eps\sign(u(j)))]\right.\\
& \left. - 2\left( 
\Phi(\frac{(\frac{\eps}{\theta(j)}-\tau)\sqrt{n}}{\sqrt{1-\tau^2}}) 
+ \Phi(-\frac{(\tau+\frac{\eps}{\theta(j)})\sqrt{n}}{\sqrt{1-\tau^2}})
- 2\Phi(-\frac{\tau\sqrt{n}}{\sqrt{1-\tau^2}}) \right) \right|
\le \frac{8C_0 }{\sqrt{n}} \cdot \frac{\tau 
	^2+1}{\sqrt{1-\tau ^2}}\,,
 \end{align*}
 where $ \Phi $ denotes the CDF of the standard 
 normal distribution and $ C_0 \le 
 \frac{\sqrt{10}+3}{6\sqrt{2\pi}}\approx 0.4097 $. If we define $ 
 \phi(x,\delta) = 2\Phi(x) - \Phi(x(1+\delta)) - 
 \Phi(x(1-\delta)) $, using the relation $ \Phi(-x) = 1-\Phi(x) $, we have 
 \[ 
 \Phi(\frac{(\frac{\eps}{\theta(j)}-\tau)\sqrt{n}}{\sqrt{1-\tau^2}}) 
 + \Phi(-\frac{(\tau+\frac{\eps}{\theta(j)})\sqrt{n}}{\sqrt{1-\tau^2}})
 - 2\Phi(-\frac{\tau\sqrt{n}}{\sqrt{1-\tau^2}}) = 
 \phi\left (\frac{\tau\sqrt{n}}{\sqrt{1-\tau^2}}, 
 \frac{\eps}{\theta(j)\tau}\right )\,.
  \]
	Therefore, we obtain 
	\[ 
	\left|
	\bE_{u(j)} [\sign(u(j)) - 
	\sign(u(j)-\eps\sign(u(j)))] - 2\phi\left 
	(\frac{\tau\sqrt{n}}{\sqrt{1-\tau^2}}, 
	\frac{\eps}{\theta(j)\tau}\right ) \right| \le \frac{8C_0 }{\sqrt{n}} \cdot 
	\frac{\tau 
		^2+1}{\sqrt{1-\tau ^2}}\,.
	 \]
	 If we set
	 \begin{equation}\label{eq:sn}
	 	s_n = 2W\tau \sum_{j\in [d]:\theta(j)>0}\theta(j)
	 	\phi\left 
	 	(\frac{\tau\sqrt{n}}{\sqrt{1-\tau^2}}, 
	 	\frac{\eps}{\theta(j)\tau}\right )\,,
	 \end{equation}
	   we have
	 \begin{equation}\label{eq:gn_sn_bound}
	 	\left | g_n - s_n \right | \le  
	 	\frac{8C_0W\tau \| \theta \|_1(\tau 
	 		^2+1)}{\sqrt{n}\sqrt{1-\tau ^2}}\,.
	 \end{equation}
	\cref{lem:phi} implies that $ \lim_{n\to\infty} s_n = 2W\tau \sum_{j\in 
	[d]:\theta(j)>0} \theta(j) H\left( \frac{\eps}{\theta(j)\tau} - 1 \right) 
	$. Since we have shown in \eqref{eq:gn_sn_bound} that $ |g_n-s_n|\le 
	O\left( \frac{1}{\sqrt{n}} 
	\right) $, we deduce $ \lim_{n\to\infty} g_n = 2W\tau \sum_{j\in 
	[d]:\theta(j)>0} \theta(j) H\left( \frac{\eps}{\theta(j)\tau} - 1 \right) $.

	If $ \frac{\eps}{\tau} < \min_{j\in [d]:\theta(j)>0} 
	\theta(j) $, we have for $ \forall j\in [d] $ such that $ \theta(j) > 0 $, 
	it holds that
	$ \frac{\eps}{\theta(j)\tau} < 1 $. \cref{lem:phi} implies that $ \phi\left 
	(\frac{\tau\sqrt{n}}{\sqrt{1-\tau^2}}, 
	\frac{\eps}{\theta(j)\tau}\right ) $ is strictly increasing in $ n $ when 
	\[ \frac{\tau\sqrt{n}}{\sqrt{1-\tau^2}} < \sqrt{\max\left \{ 
		\sfrac{3}{2},2\log\frac{1}{1-\eps/(\theta(j)\tau)} \right \}}\,, \]
	 or equivalently 
	when $ n < \left( \frac{1}{\tau^2}-1 \right) \max\{ 
\sfrac{3}{2},2\log\frac{1}{1-\eps/(\theta(j)\tau)} \} $. \cref{lem:phi} also 
implies that it is  
strictly 
	decreasing when \begin{equation*}
	\frac{\tau\sqrt{n}}{\sqrt{1-\tau^2}} \ge \sqrt{K_0 + 2\log 
	\frac{1}{1-\eps/(\theta(j)\tau)}}\,,
	\end{equation*} 
	or equivalently when $ n \ge \left( \frac{1}{\tau^2}-1 \right) \left( K_0 + 
	2\log \frac{1}{1-\eps/(\theta(j)\tau)} \right) $, where $ K_0 $ is a 
	universal constant. 
	 In light of the relation 
	between $ s_n $ and $ \phi\left 
	(\frac{\tau\sqrt{n}}{\sqrt{1-\tau^2}}, 
	\frac{\eps}{\theta(j)\tau}\right ) $ shown in \eqref{eq:sn}, we deduce that 
	$ s_n $ is strictly increasing when $  n < \left( \frac{1}{\tau^2}-1 
	\right) \max\{ 
	\sfrac{3}{2}, 2\min_{j\in [d]:\theta(j)>0}  
	\log\frac{1}{1-\eps/(\theta(j)\tau)} \} $ and that it is strictly 
	decreasing when $ n \ge \left( \frac{1}{\tau^2}-1 \right) \left( K_0 + 
	2 \max_{j\in [d]: \theta(j)>0} \log \frac{1}{1-\eps/(\theta(j)\tau)} 
	\right) $. 
	
	If $ \frac{\eps}{\tau} \ge \| \theta \|_\infty $, we have for $ \forall 
	j\in [d] $ such that $ \theta(j)>0 $, it holds that $ 
	\frac{\eps}{\theta(j)\tau} > 1 $. \cref{lem:phi} implies that $ \phi\left 
	(\frac{\tau\sqrt{n}}{\sqrt{1-\tau^2}}, 
	\frac{\eps}{\theta(j)\tau}\right ) $ is strictly increasing in $ n $ for 
	all $ n\ge 1 $. As a consequence, the sequence $ s_n $ is strictly 
	increasing for all $ n\ge 1 $. 
\end{proof}

\section{Proof of 
\cref{ob:regression-wstd-wrob}}\label{sec:proof-wrob-regression}\lc{improve writing of the proof. }
\begin{proof}
By definition, we have

\begin{align*}
\wrob_n &{}= \argmin_{w\in \bR^d}\frac{1}{n}\sum_{i=1}^n \max_{\tilde{x}_i\in B^\infty_{x_i}(\eps)} (y_i-\langle w, \tilde{x_i} \rangle )^2.
\end{align*}
For the inner maximization, since $(y_i-\langle w, \tilde{x_i} \rangle )^2$ is a convex function of $\tilde{x}_i(j)$ for each $j$, $\tilde{x}_i (j)$ is given either by $\tilde{x}_i (j) = x_i(j) + \eps$ or by $\tilde{x}_i (j) = x_i(j) - \eps$. 
We then have
\begin{align*}
    \max_{\tilde{x}_i\in B^\infty_{x_i}(\eps)} (y_i-\langle w, \tilde{x_i} \rangle)^2 & = \left(  y_i-\langle w, x_i \rangle + \sign(y_i-\langle w, x_i \rangle) \eps \sum_{j=1}^d |w(j)|        \right)^2 \\
    & = \left( |y_i - \langle w , x_i \rangle | + \eps \sum_{j=1}^d |w(j)|   \right)^2.
\end{align*}

Combining the two equations above proves \cref{ob:regression-wstd-wrob}. 
\end{proof}

\section{Proof of \cref{ob:regression-gn}}\label{sec:proof-regression-gn}\lc{improve writing of the proof}

\begin{proof}
By definition, the generalization error of a linear regression estimator $w_n$ is  
\begin{align*}
L(w_n) &{}= \bE_{x\sim P_x, \delta \sim \cN(0,\sigma^2)}(\langle w^*-w_n, x\rangle + \delta)^2  \\  
&{}= \bE_{x\sim P_X}[\langle w^*-w_n, x\rangle]^2 + \bE_{\delta \sim \cN(0,\sigma^2)}[\delta^2] \\
&{}= \bE_{x\sim P_X} (w^*-w_n)^\top xx^\top (w^*-w_n) + \sigma^2 \\
&{}= (w^*-w_n)^\top \bE_{x\sim P_X}[xx^\top] (w^*-w_n)+ \sigma^2\\
&{}= \| w_n - w^* \|^2_{\bE_{x\sim P_X}[xx^\top]}+ \sigma^2,
\end{align*}
where the second equation holds because of the independence between $x$ and $\delta$, and $\bE[\delta]=0$.

Applying this equation to both $\wrob$ and $\wstd$, we have
\begin{align*}
g_n &{}= L(\wrob_n) - L(\wstd_n) \\
&{}= \| \wrob_n - w^* \|^2_{\bE_{x\sim P_X}[xx^\top]} - \| \wstd_n - w^* 
	\|^2_{\bE_{x\sim P_X}[xx^\top]}\,.
\end{align*}
\end{proof}

\section{Proof of \cref{thm:g1-minus-infty}}\label{sec:proof-g1-minus-infty}
\begin{proof}
The standard estimator with one sample is given by $\wstd_1 = y_1/x_1$. By \cref{ob:regression-gn} we can compute
\begin{align*}
    \E\left[ (\wstd_1-w^*)^2 \right] ={}&\bE \left[ \left( \frac{y_1}{x_1} - 
	w^* 
	\right)^2 \right]
	 \\={}& \bE_{x_1, \delta\sim \cN(0,1)}\left[ \left( 
	\frac{\delta}{x_1} \right)^2 \right]
	\\={}&\bE_{\delta\sim \cN(0,1)} \left[ \delta^{2}\right] \cdot \bE_{x_1\sim \cN(0,1)} \left[ x_1^{-2}\right] 
	\\={}& 1 \cdot \int_{-\infty}^\infty \frac{1}{\sqrt{2 \pi}}  \frac{1}{x^2} e^{-\left(x^2 \right)/2} \  dx 
	\\={}& \infty
\end{align*} where the second equality is by the assumption of the model, the third equality is by Tonelli's Theorem \citep{rudin1964principles} and the independence of $x_1$ and the noise $\delta$, the fourth equality follows from the density function of standard Normal, and the last equality holds because the integration of $1/x^2$ is infinite in any neighborhood around zero. 

For the robust estimator we have 
\begin{align*}
    \wrob_1= \argmin_{w \in \bR} \left(  |y_1 - x_1 w |+ \eps |w|\right)^2 = \argmin_w  |y_1 - x_1 w |+ \eps |w|, 
\end{align*} and the minimizer is given by 
\begin{align*}
    \wrob_1 = \begin{cases}
    \frac{y_1}{x_1} &{} \textnormal{if} \ |x_1| \geq \eps
    \\0 &{} \textnormal{if} \ |x_1| < \eps
    \end{cases}.
\end{align*}Therefore we can compute
\begin{align*}
    \E \left[ (\wrob_1-w^*)^2 \right] ={}& \E_{x_1 , \delta \sim \cN(0,1)} \left[ (\wrob_1-w^*)^2 \right]
    \\ ={}& \int_{|x| < \eps} \E_{\delta \sim \cN(0,1)} \left[ \left(0-w^*\right)^2 \right] \frac{1}{\sqrt{2 \pi}} e^{- x^2 /2} \ dx
    \\ {}& + \int_{|x| \geq \eps} \E_{\delta \sim \cN(0,1)} \left[ \left(\frac{y_1}{x}-w^*\right)^2 \right] \frac{1}{\sqrt{2 \pi}} e^{- x^2 /2} \ dx
    \\ ={}& \int_{|x| < \eps} \left(w^*\right)^2 \frac{1}{\sqrt{2 \pi}} e^{- x^2 /2} \ dx
    \\ {}& + \int_{|x| \geq \eps} \E_{\delta \sim \cN(0,1)} \left[ \left(\frac{\delta}{x}\right)^2 \right] \frac{1}{\sqrt{2 \pi}} e^{- x^2 /2} \ dx
    \\ ={}& (w^*)^2 \int_{|x| < \eps} \frac{1}{\sqrt{2 \pi}} e^{- x^2 /2} \ dx + \int_{|x| \geq \eps} \frac{1}{x^2} \frac{1}{\sqrt{2 \pi}} e^{- x^2 /2} \ dx
    \\ \leq{}& (w^*)^2 \int_{|x| < \eps} \frac{1}{\sqrt{2 \pi}} e^{- x^2 /2} \ dx + \frac{1}{\eps^2}\int_{|x| \geq \eps}  \frac{1}{\sqrt{2 \pi}} e^{- x^2 /2} \ dx
    \\ \leq{}& (w^*)^2 + \frac{1}{\eps^2} < \infty
\end{align*}where the second equality is again by Tonelli's Theorem, the third equality is by the assumption of the model, the fourth equality is by distribution of $\delta$, the fifth inequality holds since $1/x^2 \leq 1/\eps^2$ for $|x|\geq\eps$, and the sixth inequality holds since the integral of the density function is less than or equal to 1.

Altogether we have $g_1 = \left(\E \left[ (\wrob_1-w^*)^2 \right] - \E\left[ (\wstd_1-w^*)^2 \right]\right)\bE_{x\sim \cN(0,1)}[x^2] = - \infty$.

\end{proof}

\section{Proof of \cref{thm:g1-minus-infty-poisson}}\label{sec:proof-g1-minus-infty-poisson}
\begin{proof}
As shown in \cref{sec:proof-g1-minus-infty}, we have $\wstd_1 = y_1/x_1$, and 
\begin{align*}
    \wrob_1 = \begin{cases}
    \frac{y_1}{x_1} &{} \text{if }  |x_1| \geq \eps\,,
    \\0 &{} \text{if }  |x_1| < \eps \,.
    \end{cases}
\end{align*}

As a result, when $x_1 \geq \eps$, we have $\wstd_1=\wrob_1$. In order to obtain the cross generalization gap, we only need to consider cases where $x_1 < \eps$. Specifically, in this case, we have
\begin{align*}
(\wrob_1-w^*)^2 - (\wstd_1-w^*)^2 &{}= (0-w^*)^2 - (y_1/x_1-w^*)^2 \\
&{}= (w^*)^2 - (\frac{w^*x_1+\delta}{x_1}-w^*)^2  \\
&{}= (w^*)^2 - \frac{\delta^2}{x_1^2}.
\end{align*}
Therefore, 
\begin{align*}
g_1 &{}= \left(\E_{(x,y)\sim \mathcal{D}} \left[ (\wrob_1-w^*)^2  -  (\wstd_1-w^*)^2 \right]\right)\bE_{x\sim \pois(\lambda)+1}[x^2] \\
&{}= \bE_{x\sim \pois(\lambda)+1}[x^2] \cdot \sum_{1\leq k < \eps}\Pr[x_1=k] \bE[(w^*)^2-\frac{\delta^2}{k^2}] \\
&{}= \bE_{x\sim \pois(\lambda)+1}[x^2] \cdot \sum_{1\leq k < \eps}\Pr[x_1=k] [(w^*)^2-\frac{1}{k^2}]
\end{align*}

Since by assumption $|w^*|\geq 1$, we have $[(w^*)^2-\frac{1}{k^2}] \geq 0,\ \forall k \geq 1$. Thus $g_1$ is non-negative, and also an increasing function with respect to $\eps \geq 0$. We note that  for $0 \leq \eps \leq 1$, we have $\wstd_1=\wrob_1$, thus $g_1 = 0$. 

Also, since $[(w^*)^2-\frac{1}{k^2}] \leq (w^*)^2$, we have
$$g_1 \leq \bE_{x\sim \pois(\lambda)+1}[x^2] \cdot \sum_{1\leq k < \eps}\Pr[x_1=k] (w^*)^2 \leq \bE_{x\sim \pois(\lambda)+1}[x^2] \cdot (w^*)^2 < \infty.$$
\end{proof}

\section{Test Loss of Standard and Robust models}
The focus of this paper is the cross generalization gap between the standard and 
adversarially robust models, which is defined as the difference between the 
test loss of two models. To further demonstrate this gap, we empirically study 
the test loss of both standard and adversarially robust models and illustrate 
the test error versus the size of the training dataset. Let $ n $ denote the 
size of the training dataset in the sequel.  

\begin{figure}[htb]
	\begin{subfigure}{0.5\textwidth}
		\centering
		\includegraphics[width=\linewidth]{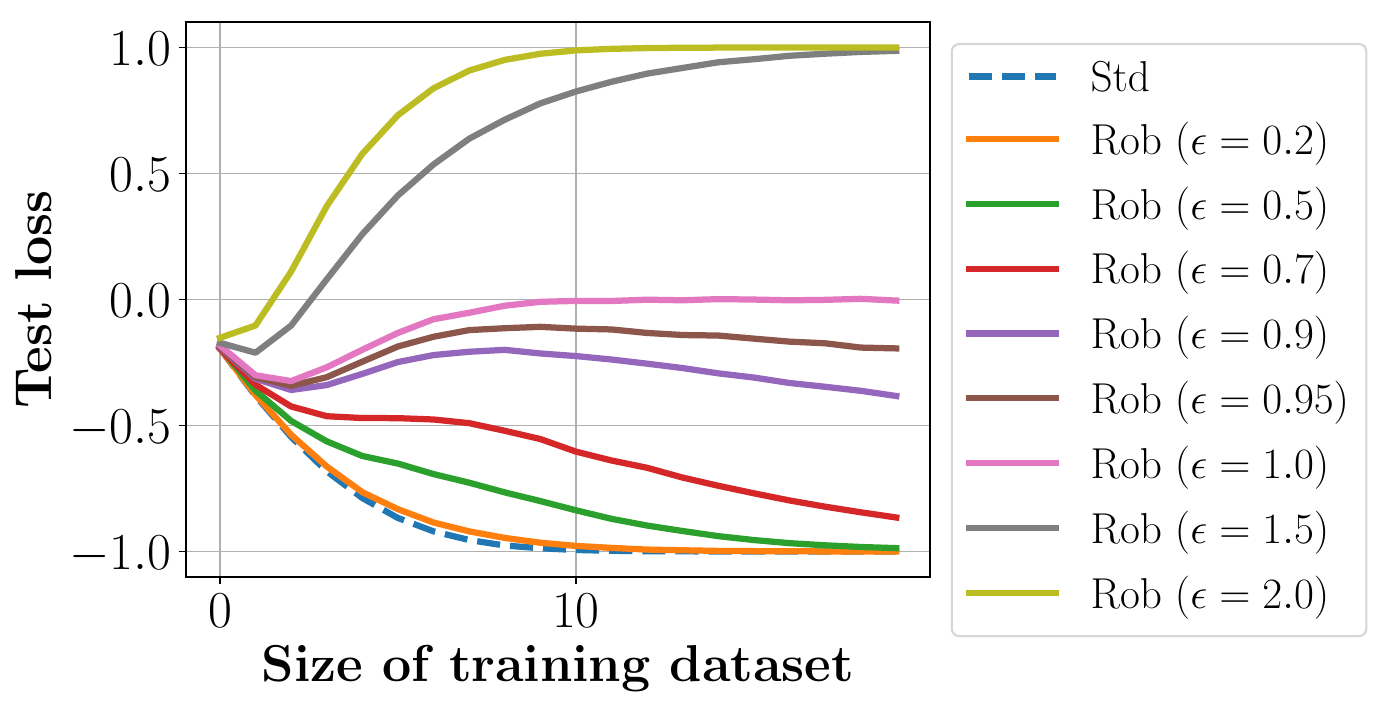}
		\caption{Gaussian model}
		\label{fig:loss-gaussian}
	\end{subfigure}
	\begin{subfigure}{0.5\textwidth}
		\centering
		\includegraphics[width=\linewidth]{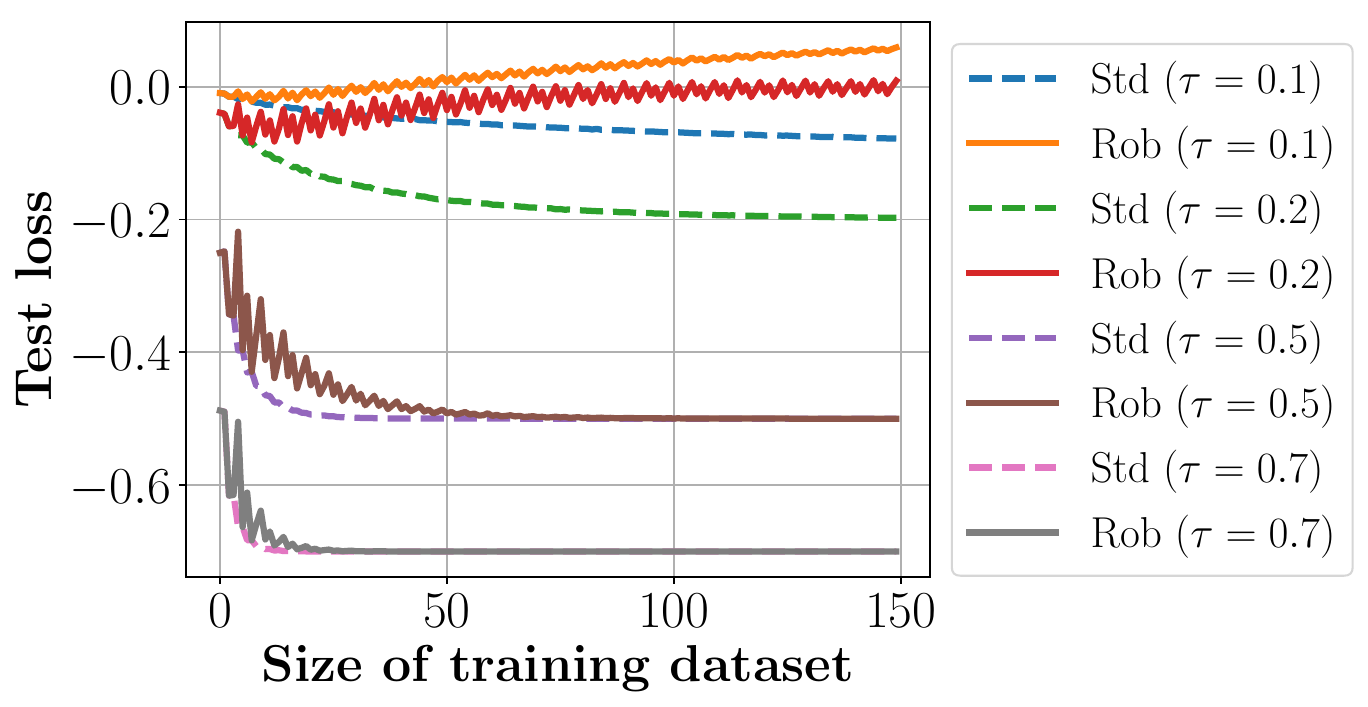}
		\caption{Bernoulli model}
		\label{fig:loss-bernoulli}
	\end{subfigure}
	\caption{Test loss vs.\ the size of the training dataset under the Gaussian 
	and Bernoulli model in the classification problem.}
	\label{fig:loss}
\end{figure}

In \cref{fig:loss}, we plot the test loss (also known as the generalization 
error) versus the size of the training dataset under the Gaussian and Bernoulli 
data generation model in the classification problem  studied in 
\cref{sec:classification}. The test loss is defined in 
\cref{sub:problem_setup}. All the model parameters are set to be identical to 
those in \cref{fig:classification}. 
In \cref{fig:loss-gaussian,fig:loss-bernoulli}, dashed curves represent the 
test  loss of the standard model. Each solid curve represents the test loss of 
an adversarially robust model with a different $\eps$. 
The cross generalization gap illustrated in \cref{fig:gaussian,fig:bernoulli} is given by 
the difference between 
the curves of the corresponding adversarially robust model and the standard 
model in 
\cref{fig:loss-gaussian,fig:loss-bernoulli}, respectively.

\cref{fig:loss-gaussian} shows the Gaussian data generation model. We observe 
that 
 the test loss of the standard model converges to $ -1 $ quickly as the size of 
 the training dataset
increases.
 The threshold between the strong and the weak 
adversary regimes is marked by $\eps = 1$. We can see that for $\eps >1$, the 
test loss monotonically approaches $ 1 $ with the size of the training dataset 
varying from $ 1 $ to $ 20 $. In this regime, more training data hurts the 
generalization of the robust model.
In the weak adversary regime, we have three observations. First, 
in general, the loss will eventually 
decrease and go towards $ -1 $. This indicates that the robust models in this 
regime will reach the standard model in terms of the test loss in the infinite 
data limit. 
Second, their
convergence to $ -1 $ is slower with an $\eps$ larger and more close to the 
threshold ($ \eps=1 $). The curve that corresponds to $\eps=1$ converges to $ 1 
$ and therefore the cross generalization gap tends to $ 1 $ for $ \eps=1 $.
 Third, 
the test loss is not 
necessarily monotonically decreasing in the training dataset size $n$. 
Particularly, for $\eps=0.9$ and $\eps=0.95$, 
the losses decrease at the initial and final stage and exhibit an increasing 
trend at the intermediate stage. During this intermediate increasing stage,   
more training data actually hurts the 
generalization of the adversarially robust model. 

\begin{figure}[tbh]
	\begin{subfigure}{0.5\textwidth}
		\centering
		\includegraphics[width=\linewidth]{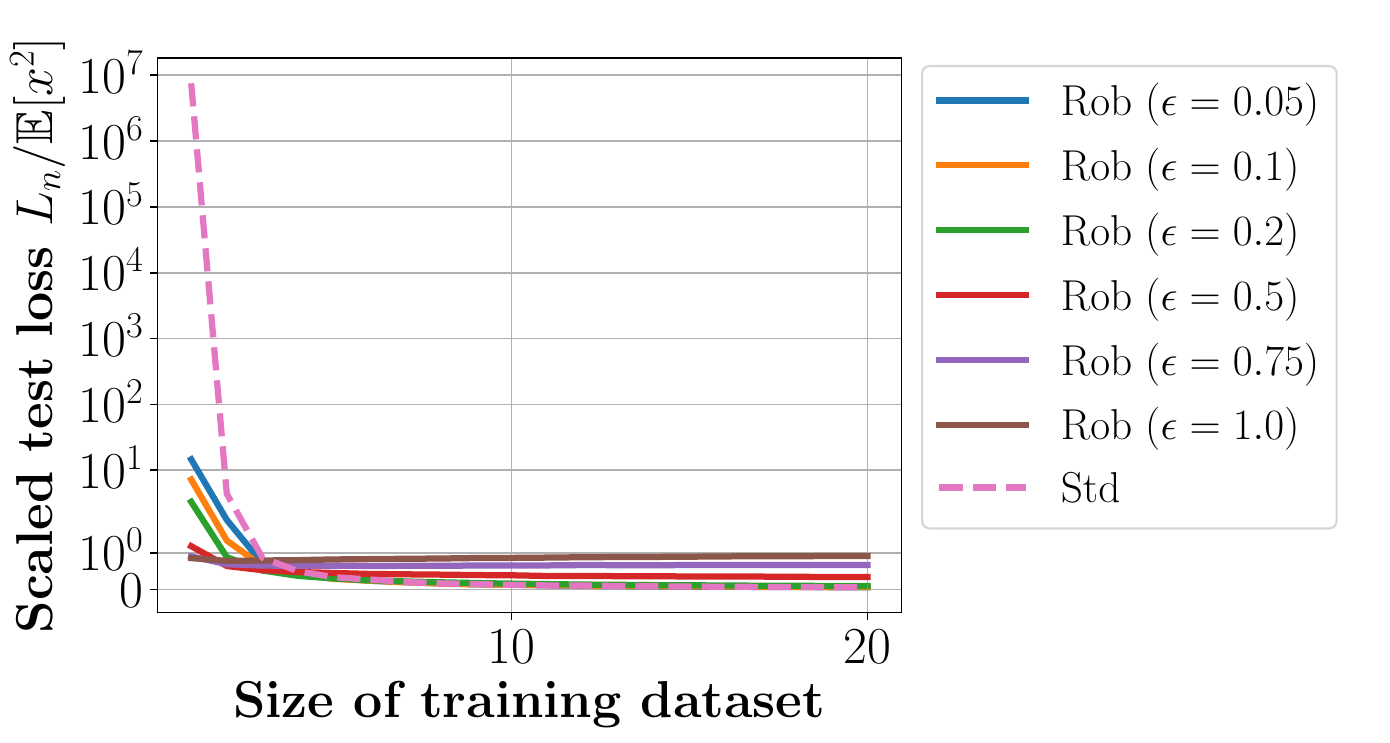}
		\caption{$ x\sim \cN(0,1) $, $ 1\le n\le 20 $.}
		\label{fig:regrloss-gauss}
	\end{subfigure}
	\begin{subfigure}{0.5\textwidth}
		\centering
		\includegraphics[width=\linewidth]{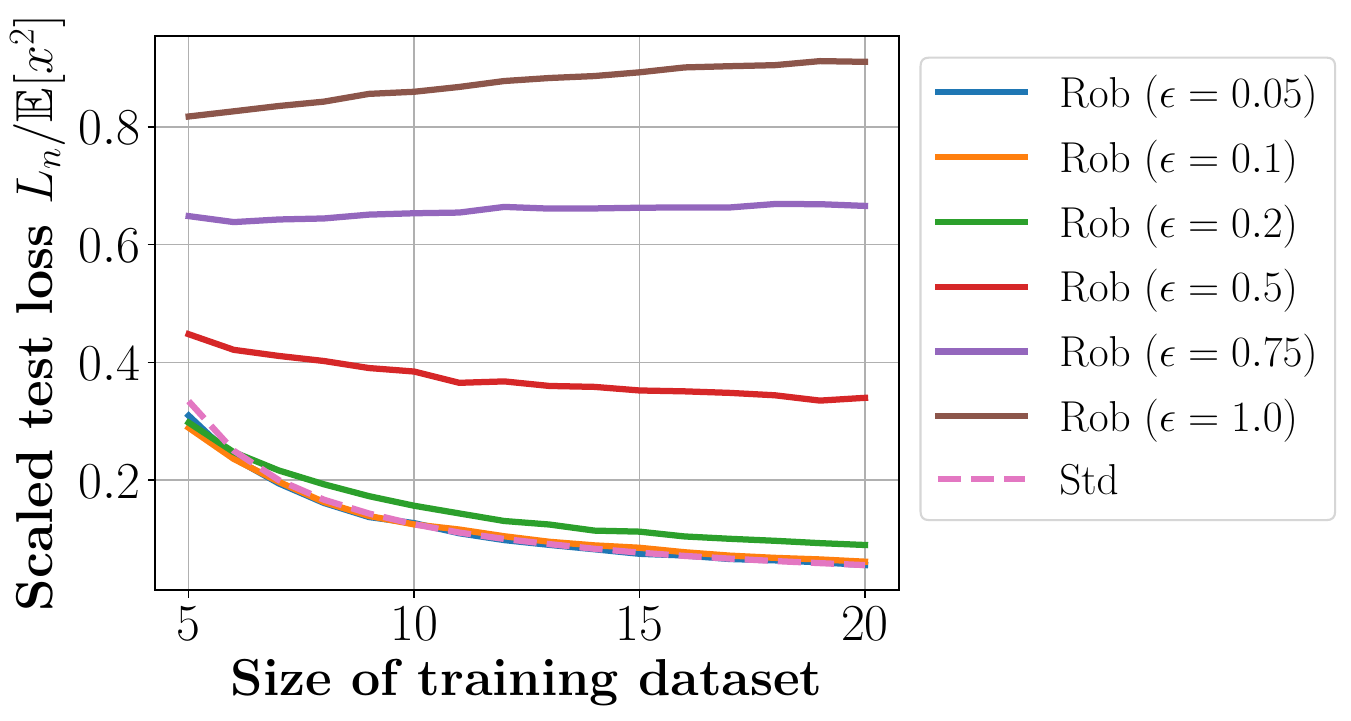}
		\caption{$ x\sim \cN(0,1) $, $ 5\le n\le 20 $.}
		\label{fig:regrloss-gauss-start5}
	\end{subfigure}
	\begin{subfigure}{0.5\textwidth}
		\centering
		\includegraphics[width=\linewidth]{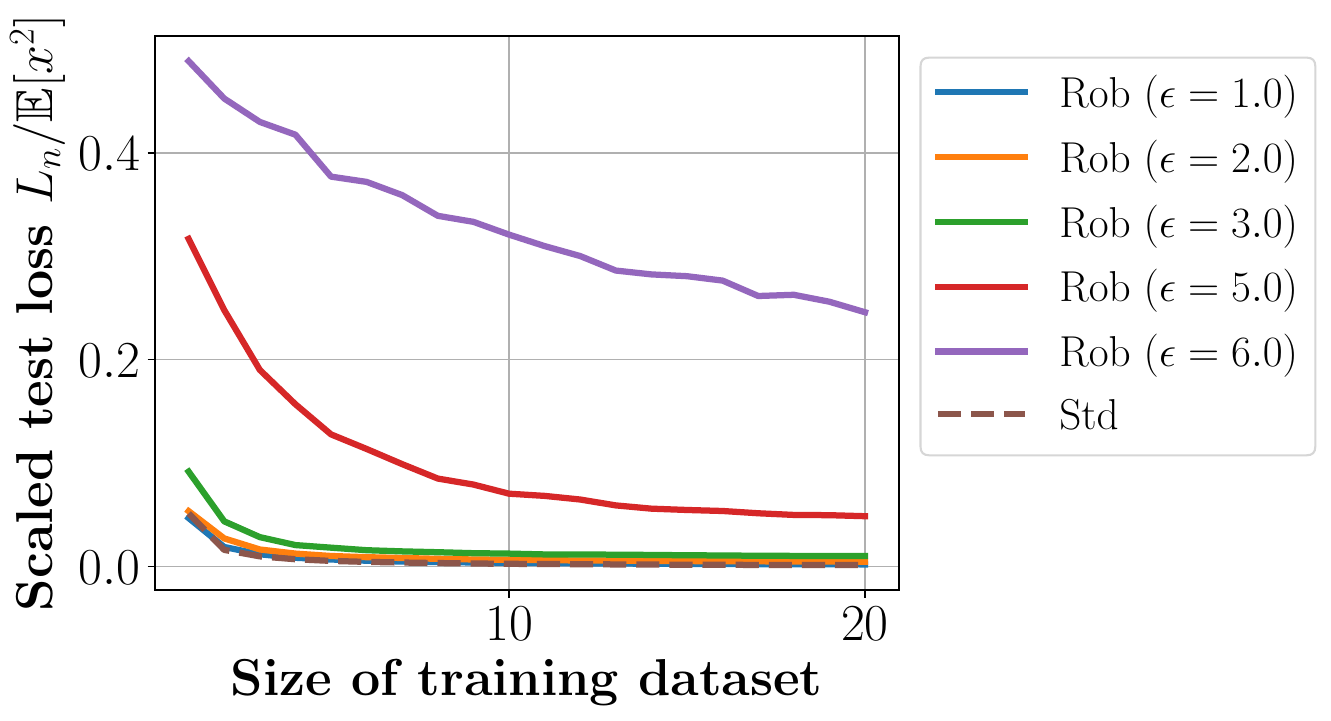}
		\caption{$ x\sim \pois(5) + 1 $, small $ \eps $.}
		\label{fig:regrloss-pois-small}
	\end{subfigure}
	\begin{subfigure}{0.5\textwidth}
		\centering
		\includegraphics[width=\linewidth]{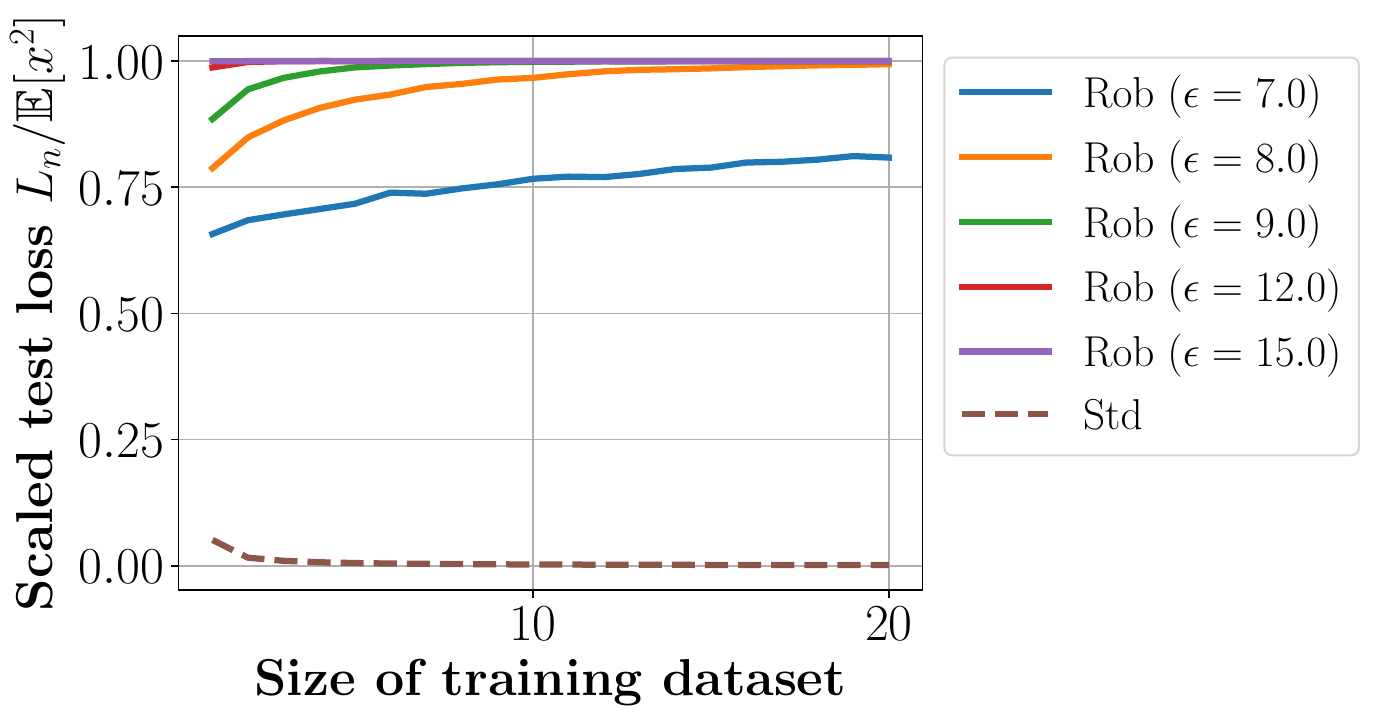}
		\caption{$ x\sim \pois(5) + 1 $, large $ \eps $.}
		\label{fig:regrloss-pois-large}
	\end{subfigure}
	\caption{Scaled test loss $ L_n/\bE_{x\sim 
			P_X}[x^2] $ vs.\ the size of the training dataset 
		(denoted by $ n $) in the linear regression problem. First two plots 
		correspond to $x$ being sampled 
		from the standard normal distribution $\cN(0,1)$ and last two plots 
		correspond to $\pois(5) + 1$. Each curve in a plot represents a 
		different choice of $\eps$. }
	\label{fig:loss-regr}
\end{figure}

The result of the Bernoulli data generation model is shown in 
\cref{fig:loss-bernoulli}. 
Recall that the 
values $ \tau = 0.1 $ and $ \tau = 0.2 $ lie in the strong adversary regime, 
while the values $ \tau = 0.5 $ and $ \tau = 0.7 $ belong to the weak adversary 
regime. 
In the weak adversary regime ($\tau = 
0.5$ and $\tau =0.7$), the test loss of the standard and robust models declines 
with more training data and tends to the same 
limit, resulting in a zero cross generalization gap. In the strong adversary regime 
($\tau = 0.1$ and $\tau=0.2$), with more training data, the test 
loss of the standard models decreases, while the test loss of the robust models 
increases. As a result, it results in an expanding cross generalization gap as 
presented in 
\cref{fig:bernoulli}.

In \cref{fig:loss-regr}, we illustrate the scaled test loss versus the size of 
the training dataset for the linear regression model that we considered in 
\cref{sec:regression}. The model parameters are identical to those in  
\cref{fig:regression}. \cref{fig:regrloss-gauss} shows the result of the 
Gaussian data model and \cref{fig:regrloss-gauss-start5} is a magnified plot of 
the same result (for $ n\ge 5 $). We observe that the test loss of the standard 
model converges to zero quickly. Regarding the robust models, for $ \eps $ 
values less than $ 0.75 $, the test loss decreases to zero with more training 
data, and it declines more quickly for smaller $ \eps $ values.
The test loss increases with more data if $ \eps $ is large ($ \eps =1.0 $); in 
this case, more training data again hurts the generalization of robust models 
in the linear regression problem.

   For the Poisson data (\cref{fig:regrloss-pois-small} and 
   \cref{fig:regrloss-pois-large}), the test loss of the standard model 
   declines with more training data and tends to zero. 
   Regarding the robust models, recall that $ \eps $ values less than or equal 
   to $  6.0 $ belong to the weak adversary regime, while the remaining $ \eps 
   $ values (\ie, those greater than $ 6.0 $) belong to the strong adversary 
   regime. 
   In the weak adversary regime, the test loss of robust models 
   decreases with more training data. In contrast, the test loss exhibits an 
   increasing trend in the strong adversary regime and thus the generalization 
   is hurt by more training data.

\end{appendices}

\end{document}